%% file: main.tex
\documentclass{article}

\usepackage[final,nonatbib]{neurips_2022}

\usepackage[utf8]{inputenc} 
\usepackage[T1]{fontenc}    
\usepackage{hyperref}       
\usepackage{url}            
\usepackage{booktabs}       
\usepackage{amsfonts}       
\usepackage{nicefrac}       
\usepackage{microtype}      
\usepackage{xcolor}         

\usepackage{overpic}
\usepackage[normalem]{ulem} 
\usepackage[ruled,algo2e]{algorithm2e}  
\usepackage{comment}
\usepackage{amsthm}
\usepackage{graphicx, wrapfig}
\usepackage{amssymb}
\usepackage{amsmath}
\usepackage{ifthen}
\newtheorem{theorem}{Theorem}

\definecolor{rebuttalCol}{HTML}{0033DD}
\newcommand{\rebuttal}[1]{#1}

\title{Scale-invariant Learning by Physics Inversion}

\author{%
  Philipp Holl \thanks{Corresponding author, \texttt{philipp.holl@tum.de}} \\
  Technical University of Munich\\
  \And
  Vladlen Koltun \\
  Apple\\
  \And
  Nils Thuerey \\
  Technical University of Munich\\
}

\begin{document}

\maketitle

\begin{abstract}
Solving inverse problems, such as parameter estimation and optimal control, is a vital part of science.
Many experiments repeatedly collect data and rely on machine learning algorithms to quickly infer solutions to the associated inverse problems.
We find that state-of-the-art training techniques are not well-suited to many problems that involve physical processes.
The highly nonlinear behavior, common in physical processes, results in strongly varying gradients that lead first-order optimizers like SGD or Adam to compute suboptimal optimization directions.
We propose a novel hybrid training approach that combines higher-order optimization methods with machine learning techniques.
We take updates from a scale-invariant inverse problem solver and embed them into the gradient-descent-based learning pipeline, replacing the regular gradient of the physical process.
We demonstrate the capabilities of our method on a variety of canonical physical systems, showing that it yields significant improvements on a wide range of optimization and learning problems.
\end{abstract}

\input{main-text}

\bibliographystyle{plain}
\bibliography{bibliography.bib}

\section*{Checklist}

\begin{enumerate}

\item For all authors...
\begin{enumerate}
  \item Do the main claims made in the abstract and introduction accurately reflect the paper's contributions and scope?
    \answerYes{}
  \item Did you describe the limitations of your work?
    \answerYes{}
  \item Did you discuss any potential negative societal impacts of your work?
    \answerNA{We do not expect any negative impacts.}
  \item Have you read the ethics review guidelines and ensured that your paper conforms to them?
    \answerYes{}
\end{enumerate}

\item If you are including theoretical results...
\begin{enumerate}
  \item Did you state the full set of assumptions of all theoretical results?
    \answerYes{We declare all assumptions the main text and provide an additional compact list in the appendix.}
        \item Did you include complete proofs of all theoretical results?
    \answerYes{}
\end{enumerate}

\item If you ran experiments...
\begin{enumerate}
\item Did you include the code, data, and instructions needed to reproduce the main experimental results (either in the supplemental material or as a URL)?
    \answerYes{We provide the full source code in the supplemental material.}
\item Did you specify all the training details (e.g., data splits, hyperparameters, how they were chosen)?
    \answerYes{We list learning rates and our training procedure in the appendix.}
\item Did you report error bars (e.g., with respect to the random seed after running experiments multiple times)?
    \answerYes{We show learning curves with different seeds in the appendix.}
\item Did you include the total amount of compute and the type of resources used (e.g., type of GPUs, internal cluster, or cloud provider)?
    \answerYes{Our hardware is listed in the appendix.}
\end{enumerate}

\item If you are using existing assets (e.g., code, data, models) or curating/releasing new assets...
\begin{enumerate}
  \item If your work uses existing assets, did you cite the creators?
    \answerYes{All used software libraries are given in the appendix.}
  \item Did you mention the license of the assets?
    \answerNA{No proprietary licenses are involved.}
  \item Did you include any new assets either in the supplemental material or as a URL?
    \answerYes{We include our source code.}
  \item Did you discuss whether and how consent was obtained from people whose data you're using/curating?
    \answerNA{}
  \item Did you discuss whether the data you are using/curating contains personally identifiable information or offensive content?
    \answerNA{}
\end{enumerate}

\item If you used crowdsourcing or conducted research with human subjects...
\begin{enumerate}
  \item Did you include the full text of instructions given to participants and screenshots, if applicable?
    \answerNA{}
  \item Did you describe any potential participant risks, with links to Institutional Review Board (IRB) approvals, if applicable?
    \answerNA{}
  \item Did you include the estimated hourly wage paid to participants and the total amount spent on participant compensation?
    \answerNA{}
\end{enumerate}

\end{enumerate}


\appendix
\newpage
\include{appendix}

\end{document}

%% file: main-text.tex
\section{Introduction}

Inverse problems that involve physical systems play a central role in computational science. This class of problems includes parameter estimation~\cite{tarantola2005inverse} and optimal control~\cite{zhou1996robust}.
Among others, solving inverse problems is integral in 
    detecting gravitational waves~\cite{RealtimeMLLIGO}, 
    controlling plasma flows~\cite{maingi2019fesreport},
    searching for neutrinoless double-beta decay~\cite{GERDA1, MAJORANA}, 
    and 
    testing general relativity~\cite{SolarEclipseDeflection, MercuryOrbit}.

%
\vspace{0.2cm} 
Decades of research in optimization have produced a wide range of iterative methods for solving inverse problems~\cite{NumericalRecipes}. Higher-order methods such as limited-memory BFGS~\cite{liu1989lbfgs} have been especially successful.
Such methods compute or approximate the Hessian of the optimization function in addition to the gradient, allowing them to locally invert the function and find stable optimization directions.
Gradient descent, in contrast, only requires the first derivative but converges more slowly, especially in ill-conditioned settings \cite{saarinen1993ill}.

Despite the success of iterative solvers, many of today's experiments rely on machine learning methods, and especially deep neural networks, to find unknown parameters given the observations~\cite{MLandScience, DL-EXO200, RealtimeMLLIGO, GERDA1}.
While learning methods typically cannot recover a solution up to machine precision, they have a number of advantages over iterative solvers.
First, their computational cost for inferring a solution is usually much lower than with iterative methods. 
This is especially important in time-critical applications, 
such as the search for rare events in data sets comprising of billions of individual recordings for collider physics.
Second, learning-based methods do not require an initial guess to solve a problem.
With iterative solvers, a poor initial guess can prevent convergence to a global optimum or lead to divergence (see Appendix~C.1). 
%
Third, learning-based solutions can be less prone to finding local optima than iterative methods because the parameters are shared across a large collection of problems~\cite{phiflow}. 
Combined with the stochastic nature of the training process, this allows gradients from other problems to push a prediction out of the basin of attraction of a local optimum.


Practically all state-of-the-art neural networks are trained using only first order information, mostly due to the computational cost of evaluating the Hessian w.r.t. the network parameters.
Despite the many breakthroughs in the field of deep learning, the fundamental shortcomings of gradient descent persist and are especially pronounced when optimizing non-linear functions such as physical systems. In such situations, the gradient magnitudes often vary strongly from example to example and parameter to parameter.

In this paper, we show that 
inverse physics solvers can be embedded into the traditional deep learning pipeline, resulting in a hybrid training scheme that aims to combine the fast convergence of higher-order solvers with the low computational cost of backpropagation for network training.
Instead of using the adjoint method to backpropagate through the physical process, we replace that gradient by the update computed from a higher-order solver which can encode the local nonlinear behavior of the physics.
These physics updates are then passed on to a traditional neural network optimizer which computes the updates to the network weights using backpropagation.
Thereby our approach maintains compatibility with acceleration schemes~\cite{duchi2011AdaGrad,Adam} and stabilization techniques~\cite{BatchNorm, StochasticDepth, LayerNorm} developed for training deep learning models.
The replacement of the physics gradient yields a simple mathematical formulation that lends itself to straightforward integration into existing machine learning frameworks.

%

In addition to a theoretical discussion, we perform an extensive empirical evaluation on a wide variety of inverse problems including the highly challenging Navier-Stokes equations.
We find that using higher-order or domain-specific solvers can drastically improve convergence speed and solution quality compared to traditional training without requiring the evaluation of the Hessian w.r.t. the model parameters.

\section{Scale-invariance in Optimization}

We consider unconstrained inverse problems that involve a differentiable physical process $\mathcal P: X \subset \mathbb R^{d_x} \rightarrow Y \subset \mathbb R^{d_y}$ which can be simulated.
Here $X$ denotes the physical parameter space and $Y$ the space of possible observed outcomes.
Given an observed or desired output $y^* \in Y$, the inverse problem consists of finding optimal parameters
\begin{equation} \label{eq:inverse-problem}
    x^* = \mathrm{arg\,min}_x L(x)  \quad \mathrm{with}  \quad  L(x) = \frac 1 2 \| \mathcal P(x) - y^* \|_2^2.
\end{equation}

Such problems are classically solved by starting with an initial guess $x_0$ and iteratively applying updates $x_{k+1} = x_k + \Delta x_k$.
Newton's method~\cite{atkinson2008numana} and many related methods~\cite{broyden1970convergence, liu1989lbfgs, GaussNewton, more1978levenbergmarq, powell1970dogleg, berndt1974bhhh, conn1991srone, avriel2003nonlinear} approximate $L$ around $x_k$ as a parabola $\tilde L(x) = L(x_k) + \frac{\partial L(x_k)}{\partial x_k} (x-x_k) + \frac 1 2 H_k (x-x_k)^2$ where $H_k$ denotes the Hessian or an approximation thereof.
Inverting $\tilde L$ and walking towards its minimum with step size $\eta$ yields
\begin{equation} \label{eq:newton}
    \Delta x_k = -\eta \cdot H_k^{-1} \left(\frac{\partial L(x_k)}{\partial x_k}\right)^T.
\end{equation}
The inversion of $H$ results in scale-invariant updates, i.e. when rescaling $x$ or any component of $x$, the optimization will behave the same way, leaving $L(x_k)$ unchanged.
An important consequence of scale-invariance is that minima will be approached equally quickly in terms of $L$ no matter how wide or deep they are.

Newton-type methods have one major downside, however.
The inversion depends on the Hessian $H$ which is expensive to compute exactly, and hard to approximate in typical machine learning settings with high-dimensional parameter spaces~\cite{DeepLearningBook} and mini-batches~\cite{schraudolph2007stochastic}.

Instead, practically all state-of-the-art deep learning relies on first-order information only.
Setting $H$ to the identity in Eq.~\ref{eq:newton} yields gradient descent updates $\Delta x = -\eta \cdot \left(\frac{\partial L}{\partial x}\right)^T$ which are not scale-invariant.
Rescaling $x$ by $\lambda$ also scales $\Delta x$ by $\lambda$, inducing a factor of $\lambda^2$ in the first-order-accurate loss change $L(x) - L(x+\Delta x) = -\eta \cdot (\frac{\partial L}{\partial x})^2 + \mathcal O(\Delta x^2)$.
Gradient descent prescribes small updates to parameters that require a large change to decrease $L$ and vice-versa, typically resulting in slower convergence than Newton updates~\cite{OptMethodComparison}.
This behavior is the root cause of exploding or vanishing gradients in deep neural networks.
The step size $\eta$ alone cannot remedy this behavior whenever $\frac{\partial L}{\partial x}$ varies along $x$.
Figure~\ref{fig:abstract-comparison} shows the optimization trajectories for the simple problem $\mathcal P(x) = (x_1, x_2^2)$ to illustrate this problem.

\begin{wrapfigure}{R}{7.2cm}
\centering
\vspace{-5mm}
\includegraphics[width=7cm]{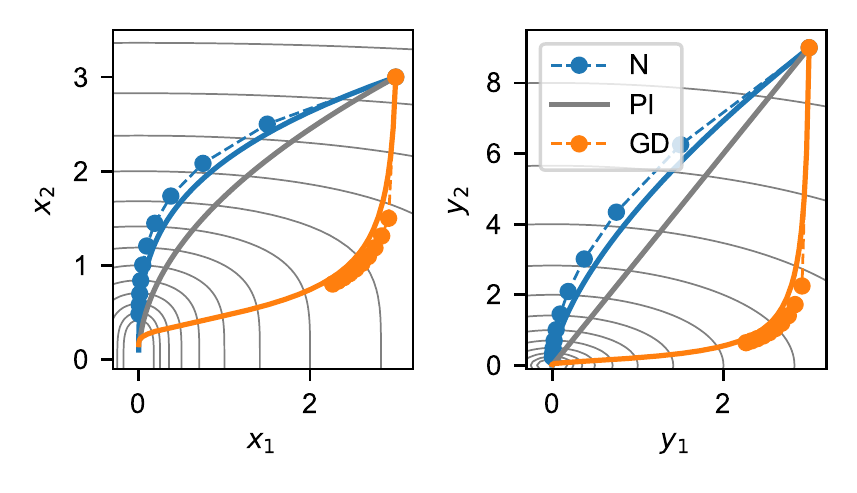}
\vspace{-5mm} 
\caption{
    Minimization (Eq.~\ref{eq:inverse-problem}) with $y \equiv \mathcal P(x) = (x_1, x_2^2)$. $L$ contours in gray. Trajectories of gradient descent (GD), Newton's method (N), and perfect physics inversion (PI) shown as lines (infinitesimal $\eta$) and circles (10 iterations with constant $\eta$).
    }
\label{fig:abstract-comparison}
\vspace{-.5cm}
\end{wrapfigure}

When training a neural network, the effect of scaling-variance can be reduced through normalization in intermediate layers~\cite{BatchNorm, LayerNorm} and regularization~\cite{AdamWeightDecay, Dropout} but this level of control is not present in most other optimization tasks, such as inverse problems (Eq.~\ref{eq:inverse-problem}).
%
More advanced first-order optimizers try to solve the scaling issue by approximating higher-order information~\cite{Adam, ConjugateGradient, duchi2011AdaGrad, HessianFree, kFAC, yao2020adahessian, pauloski2021kaisa, schnell2022half}, such as Adam where
$H \approx \mathrm{diag}\left( | \frac{\partial L}{\partial x} | \right)$, decreasing the loss scaling factor from $\lambda^2$ to $\lambda$.
However, these methods lack the exact higher-order information which limits the accuracy of the resulting update steps when optimizing nonlinear functions.

\section{Scale-invariant Physics and Deep Learning} \label{sec:PG}

We are interested in finding solutions to Eq.~\ref{eq:inverse-problem} using a neural network, $x^* = \mathrm{NN}(y^* \,|\, \theta)$, parameterized by $\theta$.
Let ${\mathcal Y^* = \{y^*_i \,|\, i = 1, ..., N\}}$ denote a set of $N$ inverse problems involving $\mathcal P$.
Then training the network means finding
\begin{equation} \label{eq:unsupervised-training}
    \theta_* = \mathrm{arg\,min}_\theta \sum_{i=1}^N \frac 1 2 \| \mathcal P\big(\mathrm{NN}(y^*_i \,|\, \theta)\big) - y^*_i \|_2^2 \,.
\end{equation}

Assuming a large parameter space $\theta$ and the use of mini-batches, higher-order methods are difficult to apply to this problem, as described above.
Additionally, the scale-variance issue of gradient descent is especially pronounced in this setting because only the network part of the joint problem $\mathrm{NN} \circ \mathcal P$ can be properly normalized while the physical process $\mathcal P$ is fixed.
Therefore, the traditional approach of computing the gradient $\frac{\partial L}{\partial \theta}$ using the adjoint method (backpropagation) can lead to undesired behavior.

Consider the problem $\mathcal P(x) = e^x$ with observed data $y^* \in (0, 1]$ (Appendix~C.2). 
Due to the exponential form of $\mathcal P$, the curvature around the solutions $x^* = \log(y^*)$ strongly depends on $y^*$.
This causes first-order network optimizers such as SGD or Adam to fail in approximating the correct solutions for small $y^*$ because their gradients are overshadowed by larger $y^*$ (see Fig.~\ref{fig:exp}).
Scaling the gradients to unit length in $x$ drastically improves the prediction accuracy, which hints at a possible solution:
If we could employ a scale-invariant physics solver, we would be able to optimize all examples, independent of the curvature around their respective minima.

\begin{wrapfigure}{R}{6cm}
\centering
\vspace{-6mm} 
\includegraphics{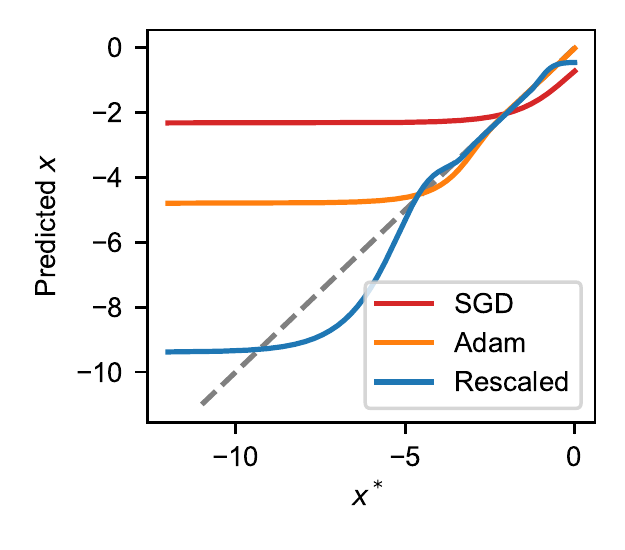} 
\vspace{-1cm} 
\caption{ 
    Networks trained according to Eq.~\ref{eq:unsupervised-training} with $\mathcal P(x) = e^x$.
    Stochastic gradient descent (SGD) and Adam fail to approximate solutions for small values due to scale variance.
    Normalizing the gradient in $x$ space (Rescaled) improves solution accuracy by decreasing scale variance.
    }
\label{fig:exp}
\vspace{-12mm}
\end{wrapfigure}

\subsection{Derivation}  
If $\mathcal P$ has a unique inverse and is sufficiently well-behaved, we can split the joint optimization problem (Eq.~\ref{eq:unsupervised-training}) into two stages.
First, solve all inverse problems individually, constructing a set of unique solutions $\mathcal X^\mathrm{sv} = \{ x^\mathrm{sv}_i = \mathcal P^{-1}(y^*_i) \}$ where $\mathcal P^{-1}$ denotes the inverse problem solver.
Second, use $\mathcal X^\mathrm{sv}$ as labels for supervised training
\begin{equation} \label{eq:supervised-training}
    \theta_* = \mathrm{arg\,min}_\theta \sum_{i=1}^N \frac 1 2 \| \mathrm{NN}(y^*_i \,|\, \theta) - x^\mathrm{sv}_i \|_2^2.
\end{equation}
This enables scale-invariant physics (SIP) inversion while a fast first-order method can be used to train the network which can be constructed to be normalized using state-of-the-art procedures~\cite{BatchNorm, LayerNorm, Dropout}.

Unfortunately, this two-stage approach is not applicable in multimodal settings, where $x^\mathrm{sv}$ depends on the initial guess $x_0$ used in the first stage.
This would cause the network to interpolate between possible solutions, leading to subpar convergence and generalization performance.
To avoid these problems, we alter the training procedure from Eq.~\ref{eq:supervised-training} in two ways.
First, we reintroduce $\mathcal P^{-1}$ into the training loop, yielding
\begin{equation} \label{eq:hybrid-training-intermediate}
    \theta_* = \mathrm{arg\,min}_\theta \sum_{i=1}^N \frac 1 2 \| \mathrm{NN}(y^*_i \,|\, \theta) - \mathcal P^{-1}(y^*_i) \|_2^2 \, .
\end{equation}
Next, we condition $\mathcal P^{-1}$ on the neural network prediction by using it as an initial guess, $\mathcal P^{-1}(y^*) \rightarrow \mathcal P^{-1}(y^* \,|\, \mathrm{NN}(y^* \,|\, \theta))$.
This makes training in multimodal settings possible because the embedded solver $\mathcal P^{-1}$ searches for minima close to the prediction of the  $\mathrm{NN}$.
Therefore $\theta$ can exit the basin of attraction of other minima and does not need to interpolate between possible solutions.
Also, since all inverse problems from $\mathcal Y^*$ are optimized jointly, this reduces the likelihood of any individual solution getting stuck in a local minimum, as discussed earlier. 


The obvious downside to this approach is that $\mathcal P^{-1}$ must be run for each training step.
When $\mathcal P^{-1}$ is an iterative solver, we can write it as $P^{-1}(y^* \,|\, \mathrm{NN}(y^* \,|\, \theta)) = \mathrm{NN}(y^* \,|\, \theta) + \Delta x_0 + ... + \Delta x_n$.
We denote the first update $\Delta x_0 \equiv U(y^* \,|\, \theta)$.
%


Instead of computing all updates $\Delta x$, we approximate $P^{-1}$ with its first update $U$.
Inserting this into Eq.~\ref{eq:hybrid-training-intermediate} with $(\circ) \equiv (y_i^* \,|\, \theta)$ yields
\begin{equation} \label{eq:hybrid-training}
    \theta_* = \mathrm{arg\,min}_\theta \sum_{i=1}^N \frac 1 2 \| \mathrm{NN}(y^*_i \,|\, \theta) 
    - \left( \mathrm{NN}(\circ) + U(\circ) \right) \|_2^2.
\end{equation}
This can be further simplified to $\sum_{i=1}^N \frac 1 2 \| U(y_i^* \,|\, \theta) \|_2^2$ but this form is hard to optimize directly as is requires backpropagation through $U$.
In addition, its minimization is not sufficient, because all fixed points of $U$, such as maxima or saddle points of $L$, can act as attractors.

Instead, we make the assumption $\frac{\partial \mathcal P^{-1}}{\partial y} = 0$ to remove the $(\circ)$ dependencies in Eq.~\ref{eq:hybrid-training}, treating them as constant.
This results in a simple $L^2$ loss for the network output.
As we will show, this avoids both issues.
It also allows us to factor the optimization into a network and a physics graph (see Fig.~\ref{fig:opt-flow}), so that all necessary derivative computations only require data from one of them.

\subsection{Update Rule}
The factorization described above results in the following update rule for training with SIP updates, shown in Fig.~\ref{fig:opt-flow}:
\begin{enumerate}
    \item Pass the network prediction $x_0$ to the physics graph and compute $\Delta x_0 \equiv U(y_i^* \,|\, x_0)$ for all examples in the mini-batch. 
    \item Send $\tilde x \equiv x_0 + \Delta x_0$ back and compute the network loss $\tilde L = \frac 1 2 || x_0 - \tilde x ||_2^2$.
    \item Update $\theta$ using any neural network optimizer, such as SGD or Adam, treating $\tilde x$ as a constant.
\end{enumerate}


\begin{figure}[htb]
\centering
\includegraphics[width=10cm]{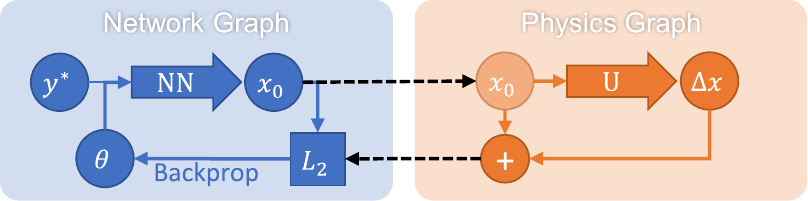}
\caption{
    Neural network (NN) training procedure with embedded inverse physics solver (U). 
    }
\label{fig:opt-flow}
\end{figure}

To see what updates $\Delta \theta$ result from this algorithm, we compute the gradient w.r.t. $\theta$,
\begin{equation}
    \frac{\partial \tilde L}{\partial \theta} = 
    \sum_{i=1}^N U(y^* \,|\, x_0) \cdot \frac{\partial \mathrm{NN}}{\partial \theta}
\end{equation}

Comparing this to the gradient resulting from optimizing the joint problem $\mathrm{NN} \circ \mathcal P$ with a single optimizer (Eq.~\ref{eq:unsupervised-training}),
\begin{equation}
    \frac{\partial L}{\partial\theta} =
    \Delta y_i \cdot \frac{\partial \mathcal P}{\partial x} \frac{\partial \mathrm{NN}}{\partial \theta}
\end{equation}
where $\Delta y_i = \mathcal P\big(\mathrm{NN}(y^*_i \,|\, \theta)\big) - y^*_i$,
we see that $U$ takes the place of $\Delta y_i \cdot \frac{\partial \mathcal P}{\partial x}$, the adjoint vector that would otherwise be computed by backpropagation.
Unlike the adjoint vector, however, $\tilde x = x_0 + U$ encodes an actual physical configuration.
Since $\tilde x$ can stem from a higher-order solver, the resulting updates $\Delta \theta$ can also encode non-linear information about the physics without computing the Hessian w.r.t. $\theta$.

\subsection{Convergence}
It is a priori not clear whether SIP training will converge, given that $\frac{\partial \mathcal P^{-1}}{\partial y}$ is not computed.
We start by proving convergence for two special choices of the inverse solver $\mathcal P^{-1}$ before considering the general case.
We assume that $\mathrm{NN}$ is expressive enough to fit our problem and that it is able to converge to every point $x$ for all examples using gradient descent:
\begin{equation} \label{assumption:network}
    \exists \eta > 0 \, : \, \forall i \  \forall x \ \forall \epsilon > 0 \ \exists n \in \mathbb N : || \mathrm{NN}_{\theta_n} - x ||_2 \leq \epsilon
\end{equation}
where $\theta_n$ is the sequence of gradient descent steps
\begin{equation}
    \theta_{n+1} = \theta_n - \eta  \left(\frac{\partial \mathrm{NN}}{\partial \theta}\right)^T(\mathrm{NN}_{\theta_n} - x)
\end{equation}
with $\eta > 0$.
Large enough networks fulfill this property under certain assumptions~\cite{du2018gradient} and the universal approximation theorem guarantees that such a configuration exists~\cite{NNApproximator}.
%

In the first special case, $U(x) \equiv U(y^* \, | \, x)$ points directly towards a solution $x^*$. This models the case of a known ground truth solution in a unimodal setting.
For brevity, we will drop the example indices and the dependencies on $y^*$.

\begin{theorem} \label{theorem:unique-proof}
If $\forall x \ \exists \lambda \in (0, 1] \  : \ U(x) = \lambda (x^* - x)$,
then the gradient descent sequence $\mathrm{NN}_{\theta_n}$ with $\theta_{n+1} = \theta_n + \eta \left(\frac{\partial \mathrm{NN}}{\partial \theta}\right)^T U(x)$ converges to $x^*$.
\end{theorem}
\begin{proof}
Rewriting $U(x) = -\frac{\partial}{\partial x} \left( \frac \lambda 2 ||x - x^*||_2^2 \right)$ yields the update $\theta_{n+1} - \theta_n = - \eta \left(\frac{\partial \hat L}{\partial x} \frac{\partial \mathrm{NN}}{\partial \theta} \right)^T$ where $\hat L = \frac \lambda 2 ||x-x^*||_2^2$. This describes gradient descent towards $x^*$ with the gradients scaled by $\lambda$.
Since $\lambda \in (0, 1]$, the convergence proof of gradient descent applies.
\end{proof}

The second special case has $U(x)$ pointing in the direction of steepest gradient descent in $x$ space.

\begin{theorem}
If $\forall x \  \exists \lambda \in (0, 1] \  : \ U(x) = - \lambda \left(\frac{\partial L}{\partial x}\right)^T$,
then the gradient descent sequence $\mathrm{NN}_{\theta_n}$ with $\theta_{n+1} = \theta_n + \eta \left(\frac{\partial \mathrm{NN}}{\partial \theta}\right)^T U(x) $ converges to minima of $L$.
\end{theorem}
\begin{proof}
This is equivalent to gradient descent in $L(\theta) \equiv (\mathrm{NN} \circ L)(\theta)$.
Rewriting the update yields $\theta_{n+1} - \theta_n
= - \eta \lambda \left(\frac{\partial L}{\partial x} \frac{\partial \mathrm{NN}}{\partial \theta}\right)^T  
$ which is the gradient descent update scaled by $\lambda$.
\end{proof}

Next, we consider the general case of an arbitrary $\mathcal P^{-1}$ and $U$.
We require that $U$ decreases $L$ by a minimum relative amount specified by $\tau$,
\begin{equation} \label{assumption:loss-decrease}
    \exists \tau > 0 \ : \ \forall x \ : \ L(x) - L(x+U(x)) \geq \tau \left(L(x) - L(x^*)\right)
    .
\end{equation}
To guarantee convergence to a region, we also require
\begin{equation} \label{assumption:settle-down}
    \exists K > 0 \ : \ \forall x \ : \ ||U(x)|| \leq K (L(x) - L(x^*))
    .
\end{equation}



\begin{theorem} \label{theorem:main-proof}
There exists an update strategy $\theta_{n+1} = S(\theta_n)$ based on a single evaluation of $U$ for which $L(\mathrm{NN}_{\theta_n}(y))$ converges to a minimum $x^*$ or minimum region of $L$ for all examples.
\end{theorem}

\begin{wrapfigure}{R}{7cm}
\centering
\vspace{-3mm}
\includegraphics[width=7cm]{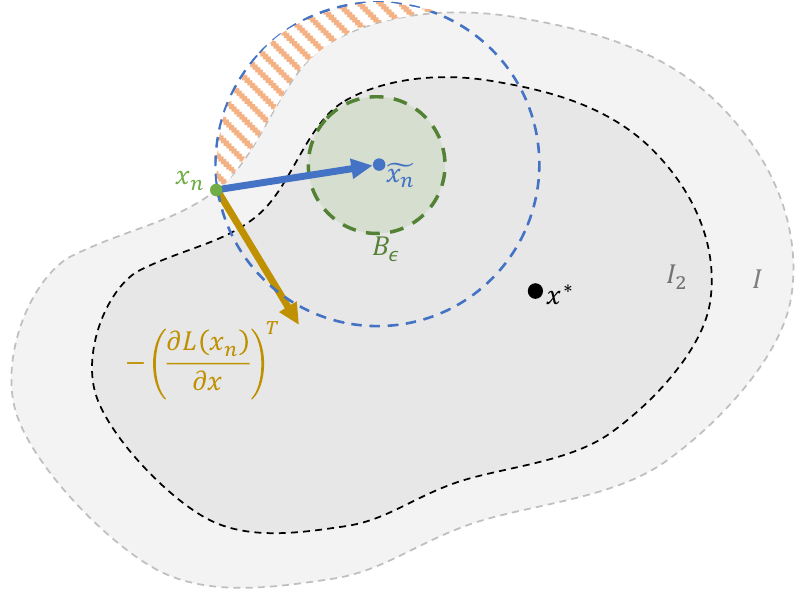}
\caption{
    Convergence visualization for the proof of theorem~\ref{theorem:main-proof}. All shown objects are visualized in $x$ space for one example $i$.
    }
\label{fig:convergence}
\vspace{-.4cm}
\end{wrapfigure}

\begin{proof}
We denote $\tilde x_n \equiv x_n + U(x_n)$ and $\Delta L^* = L(x_n) - L(x^*)$.
Let $I_2$ denote the open set of all $x$ for which $L(x) - L(x_n) > \frac \tau 2 \Delta L^*$.
Eq.~\ref{assumption:loss-decrease} provides that $\tilde x_n \in I_2$.
Since $I_2$ is open, $\exists \epsilon > 0 : \forall x \in B_\epsilon(\tilde x_n) : L(x_n) - L(x) > \frac \tau 2 \Delta L^*$ where $B_\epsilon(x)$ denotes the open set containing all $x' \in \mathbb R^d$ for which $||x' - x||_2 < \epsilon$, i.e. there exists a small ball around $\tilde x_n$ which is fully contained in $I_2$ (see sketch in Fig.~\ref{fig:convergence}).

Using Eq.~\ref{assumption:network}, we can find a finite $n \in \mathbb N$ for which $\mathrm{NN}_{\theta_n} \in B_\epsilon(\tilde x_n)$ and therefore $L(x_n) - L(\mathrm{NN}_{\theta_n}) > \frac \tau 2 \Delta L^*$.
We can thus use the following strategy $S$ for minimizing $L$:
First, compute $\tilde x_n = x_n + U(x_n)$.
Then perform gradient descent steps in $\theta$ with the effective objective function $\frac 1 2 ||\mathrm{NN}_\theta - \tilde x_n||_2^2$ until $L(x_n) - L(\mathrm{NN}_\theta) \geq \frac \tau 2 \Delta L^*$.
Each application of $S$ reduces the loss to $\Delta L_{n+1}^* \leq (1 - \frac \tau 2) \Delta L_n^*$ so any value of $L > L(x^*)$ can be reached within a finite number of steps.
Eq.~\ref{assumption:settle-down} ensures that $||U(x)|| \rightarrow 0$ as the optimization progresses which guarantees that the optimization converges to a minimum region.
\end{proof}

While this theorem guarantees convergence, it requires potentially many gradient descent steps in $\theta$ for each physics update $U$.
This can be advantageous in special circumstances, e.g. when $U$ is more expensive to compute than an update in $\theta$, or when $\theta$ is far away from a solution.
However, in many cases, we want to re-evaluate $U$ after each update to $\theta$.
Without additional assumptions about $U$ and $\mathrm{NN}_\theta$, there is no guarantee that $L$ will decrease every iteration, even for infinitesimally small step sizes.
Despite this, there is good reason to assume that the optimization will decrease $L$ over time.
This can be seen visually in Fig.~\ref{fig:convergence} where the next prediction $x_{n+1}$ is guaranteed to lie within the blue circle.
The region of increasing loss is shaded orange and always fills less than half of the volume of the circle, assuming we choose a sufficiently small step size.
We formalize this argument in appendix~A.2.

\paragraph{Remarks and lemmas}
We considered the case that $U \equiv \Delta x_0$. This can be trivially extended to the case that $U \equiv \Delta x_0 + ... + \Delta x_m$ for any $m \in \mathbb N$.
When we let the solver run to convergence, i.e. $m$ large enough, theorem~\ref{theorem:unique-proof} guarantees convergence if $\mathcal P^{-1}$ consistently converges to the same $x^*_i$.
Also note that any minimum $\theta^*$ found with SIP training fulfills $U=0$ for all examples, i.e. we are implicitly minimizing $\sum_{i=1}^N \frac 1 2 \| U(y_i^* \,|\, \theta) \|_2^2$.
%

\subsection{Experimental Characterization}
We first investigate the convergence behavior of SIP training depending on characteristics of $\mathcal P$.
We construct the synthetic two-dimensional inverse process
$$\mathcal P(x) = \left(\sin(\hat x_1) / \xi, \  \hat x_2 \cdot \xi \right) \quad \mathrm{with} \quad \hat x = \gamma \cdot R_\phi \cdot x , $$
where $R_\phi \in \mathrm{SO}(2)$ denotes a rotation matrix and $\gamma > 0$.
The parameters $\xi$ and $\phi$ allow us to continuously change the characteristics of the system.
The value of $\xi$ determines the conditioning of $\mathcal P$ with large $\xi$ representing ill-conditioned problems while $\phi$ describes the 
coupling 
of $x_1$ and $x_2$. When $\phi=0$, the off-diagonal elements of the Hessian vanish and the problem factors into two independent problems.
Fig.~\ref{fig:sin-main}a shows one example loss landscape.

We train a fully-connected neural network to invert this problem (Eq.~\ref{eq:unsupervised-training}), comparing SIP training using a saddle-free Newton solver~\cite{SaddleFreeNewton} to various state-of-the-art network optimizers.
We select the best learning rate for each optimizer independently.
For $\xi=0$, when the problem is perfectly conditioned, all network optimizers converge, with Adam converging the quickest (Fig.~\ref{fig:sin-main}b).
Note that the relatively slow convergence of SIP mostly stems from it taking significantly more time per iteration than the other methods, on average 3 times as long as Adam.
As we have spent little time in eliminating computational overheads, SIP performance could likely be significantly improved to near-Adam performance.

\begin{figure}[bt]
\centering
\includegraphics{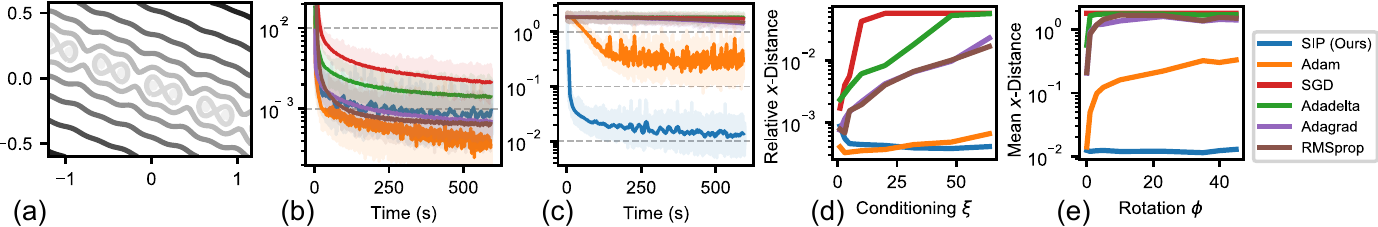}
\caption{\label{fig:sin-main}
(a)~Example loss landscape with $y^*=(0.3, -0.5)$, $\xi=1$, $\phi=15^\circ$.
(b,c)~Learning curves with $\phi=\frac{\pi}{4}$, averaged over 256 min-batches. For (b) $\xi=1$, and (c) $\xi=32$.
(d)~Dependence on problem conditioning $\xi$ (with $\phi=0$).
(e)~Dependence on parameter coupling $\phi$ (with $\xi=32$).}
\end{figure}

When increasing $\xi$ with $\phi=0$ fixed (Fig.~\ref{fig:sin-main}d), the accuracy of all traditional network optimizers decreases because the gradients scale with $(1/\xi, \xi)$ in $x$, elongating in $x_2$, the direction that requires more precise values.
SIP training uses the Hessian to invert the scaling behavior, producing updates that align with the flat direction in $x$ to avoid this issue.
This allows SIP training to retain its relative accuracy over a wide range of $\xi$.
At $\xi=32$, only SIP and Adam succeed in optimizing the network to a significant degree (Fig.~\ref{fig:sin-main}c).

Varying $\phi$ with $\xi=32$ fixed (Fig.~\ref{fig:sin-main}e) sheds light on how Adam manages to learn in ill-conditioned settings.
Its diagonal approximation of the Hessian reduces the scaling effect when $x_1$ and $x_2$ lie on different scales, but when the parameters are coupled, the lack of off-diagonal terms prevents this.
SIP training has no problem with coupled parameters since its updates are based on the full-rank Hessian $\frac{\partial^2 L}{\partial x^2}$.



\subsection{Application to High-dimensional Problems}
Explicitly evaluating the Hessian is not feasible for high-dimensional problems.
However, scale-invariant updates can still be computed, e.g. by inverting the gradient or via domain knowledge.
We test SIP training on three high-dimensional physical systems described by partial differential equations: Poisson's equation, the heat equation, and the Navier-Stokes equations.
This selection covers ubiquitous physical processes with diffusion, transport, and strongly non-local effects, featuring both explicit and implicit solvers.
%
All code required to reproduce our results is available at \url{https://github.com/tum-pbs/SIP}. 
A detailed description of all experiments along with additional visualizations and performance measurements can be found in appendix~B. 



%
%
%

\paragraph{Poisson's equation} \label{sec:poisson}

Poisson's equation, $\mathcal P(x) = \nabla^{-2} x$, plays an important role in electrostatics, Newtonian gravity, and fluid dynamics~\cite{ames2014numerical}.
It has the property that local changes in $x$ can affect $\mathcal P(x)$ globally.
Here we consider a two-dimensional system and train a U-net~\cite{UNet} to solve inverse problems (Eq.~\ref{eq:inverse-problem}) on pseudo-randomly generated $y^*$.
We compare SIP training to SGD with momentum, Adam, AdaHessian~\cite{yao2020adahessian}, Fourier neural operators (FNO)~\cite{kovachki2021neural} and Hessian-free optimization (H-free)~\cite{HessianFree}.
Fig.~\ref{fig:Poisson}b shows the learning curves.
%
The training with SGD, Adam and AdaHessian drastically slows within the first 300 iterations.
FNO and H-free both improve upon this behavior, reaching twice the accuracy before slowing.
For SIP, we construct scale-invariant $\Delta x$ based on the analytic inverse of Poisson's equation and use Adam to compute $\Delta\theta$.
The curve closely resembles an exponential curve, which indicates linear convergence, the ideal case for first-order methods optimizing an $L_2$ objective.
During all of training, the SIP variant converges exponentially faster than the traditional optimizers, its relative performance difference compared to Adam continually increasing from a factor of 3 at iteration 60 to a full order of magnitude after 5k iterations.
This difference can be seen in the inferred solutions (Fig.~\ref{fig:Poisson}a) which are noticeably more detailed.

\begin{figure}[tb]
\centering
\includegraphics{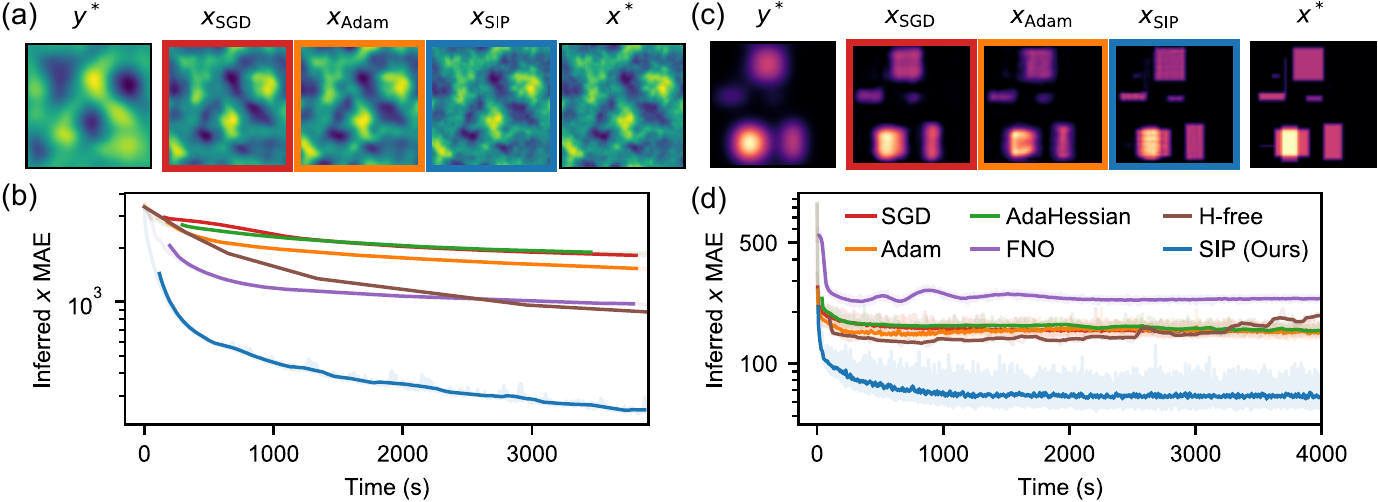}
\caption{\label{fig:Poisson} Poisson's equation (left) and the heat equation (right). (a,c)~Example from the data set: observed distribution ($y^*$), inferred solutions, and ground truth solution ($x^*$). (b,d)~Learning curves, running average over 64 mini-batches (except for H-free).}
\end{figure}

\paragraph{Heat equation} \label{sec:heat} 


Next, we consider a system with fundamentally non-invertible dynamics.
The heat equation, $\frac{\partial u}{\partial t} = \nu \cdot \nabla^2 u$, models heat flow in solids but also plays a part in many diffusive systems~\cite{DiffusionEquations}.
It gradually destroys information as the temperature equilibrium is approached~\cite{HeatInformationLoss}, causing $\nabla \mathcal P$ to become near-singular.
Inspired by heat conduction in microprocessors, we generate examples $x_\mathrm{GT}$ by randomly scattering four to ten heat generating rectangular regions on a plane and simulating the heat profile $y^* = \mathcal P(x_\mathrm{GT})$ as observed from outside a heat-conducting casing.
%
The learning curves for the corresponding inverse problem are shown in Fig.~\ref{fig:Poisson}d.

When training with SGD, Adam or AdaHessian, we observe that the distance to the solution starts rapidly decreasing before decelerating between iterations 30 and 40 to a slow but mostly stable convergence.
The sudden deceleration is rooted in the adjoint problem, which is also a diffusion problem.
Backpropagation through $\mathcal P$ removes detail from the gradients, which makes it hard for first-order methods to recover the solution.
H-free initially finds better solutions but then stagnates with the solution quality slowly deteriorating.
FNO performs poorly on this task, likely due to the sharp edges in $x^*$.

The information loss in $\mathcal P$ prevents direct numerical inversion of the gradient or Hessian.
Instead, we add a dampening term to derive a numerically stable scale-invariant solver which we use for SIP training.
Unlike SGD and Adam, the convergence of SIP training does not decelerate as early as the other methods, resulting in an exponentially faster convergence.
At iteration 100, the predictions are around 34\% more accurate compared to Adam, and the difference increases to 130\% after 10k iterations, making the reconstructions noticeably sharper than with traditional training methods (Fig.~\ref{fig:Poisson}c).
To test the dependence of SIP on hyperparameters like batch size or learning rate, we perform this experiment with a range of values (see appendix~B.3).
Our results indicate that SIP training and Adam are impacted the same way by non-optimal hyperparameter configurations.

\paragraph{Navier-Stokes equations} \label{sec:ns}


Fluids and turbulence are among the most challenging and least understood areas of physics due to their highly nonlinear behavior and chaotic nature~\cite{NavierStokesOpenQuestions}.
We consider a two-dimensional system governed by the incompressible Navier-Stokes equations: $\frac{\partial v}{\partial t} = \nu \nabla^2 v - v \cdot \nabla v - \nabla p$, $\nabla \cdot v = 0$, where $p$ denotes pressure and $\nu$ the viscosity.
At $t=0$, a region of the fluid is randomly marked with a massless colorant $m_0$ that passively moves with the fluid, $\frac{\partial m}{\partial t} = - v \cdot \nabla m$.
After time $t$, the marker is observed again to obtain $m_t$.
The fluid motion is initialized as a superposition of linear motion, a large vortex and small-scale perturbations.
An example observation pair $y^* = \{m_0, m_t\}$ is shown in Fig.~\ref{fig:Fluid}a.
The task is to find an initial fluid velocity $x \equiv v_0$ such that the fluid simulation $\mathcal P$ matches $m_t$ at time $t$.
Since $\mathcal P$ is deterministic, $x$ encodes the complete fluid flow from $0$ to $t$.
We define the objective in frequency space with lower frequencies being weighted more strongly.
This definition considers the match of the marker distribution on all scales, from the coarse global match to fine details, and is compatible with the definition in Eq.~\ref{eq:inverse-problem}.
We train a U-net~\cite{UNet} to solve these inverse problems; the learning curves are shown in Fig.~\ref{fig:Fluid}b.

When training with Adam, the error decreases for the first 100 iterations while the network learns to infer velocities that lead to an approximate match.
The error then proceeds to decline at a much lower rate, nearly coming to a standstill.
This is caused by an overshoot in terms of vorticity, as visible in Fig.~\ref{fig:Fluid}a right.
While the resulting dynamics can roughly approximate the shape of the observed $m_t$, they fail to match its detailed structure.
Moving from this local optimum to the global optimum is very hard for the network as the distance in $x$ space is large and the gradients become very noisy due to the highly non-linear physics.
A similar behavior can also be seen when optimizing single inverse problems with gradient descent where  it takes more than 20k iterations for GD to converge on single problems.

\begin{figure}[t!]
\centering
\includegraphics[width=8.0cm]{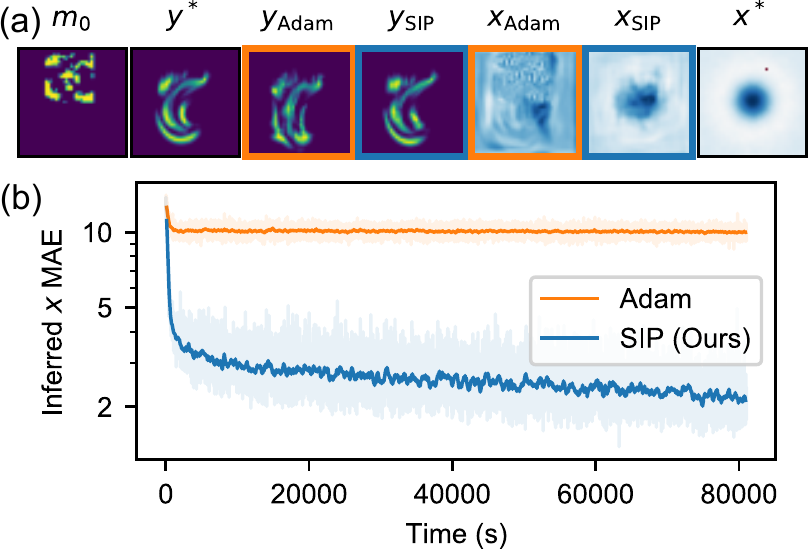}
\caption{\label{fig:Fluid} Incompressible fluid flow reconstruction. (a)~Example from the data set: initial marker distribution ($m_0$); simulated marker distribution after time $t$ using ground-truth velocity ($y^*$) and network predictions ($y_\circ$); predicted initial velocities ($x_\circ$); ground truth velocity ($x^*$).
(b)~Learning curves, running average over 64 mini-batches.}
\end{figure}

For SIP training, we run the simulation in reverse, starting with $m_t$, to estimate the translation and vortex of $x$ in a scale-invariant manner.
%
When used for network training, we observe vastly improved convergence behavior compared to first-order training.
The error rapidly decreases during the first 200 iterations, at which point the inferred solutions are more accurate than pure Adam training by a factor of 2.3.
The error then continues to improve at a slower but still exponentially faster rate than first-order training, reaching a relative accuracy advantage of 5x after 20k iterations.
To match the network trained with pure Adam for 20k iterations, the SIP variant only requires 55 iterations.
This improvement is possible because the inverse-physics solver and associated SIP updates do not suffer from the strongly-varying gradient magnitudes and directions, which drown the first-order signal in noise.
Instead, the SIP updates behave much more smoothly, both in magnitude and direction.

\begin{wraptable}{R}{7cm}
\vspace{-0.5cm}
\caption{\label{tab:fluid-speed}Time to reach equal solution quality in the fluid experiment, measured as MAE in $x$ space.
The inference time is given per example in batch mode, followed by the number of iterations in parentheses.}\vspace{0.15cm}
\begin{tabular}{ lcc } 
 \hline
 Method & Training time & Inference time \\ 
 \hline
 $\mathrm{NN}$ & 17.6 h (15.6 k) & 0.11 ms \\ 
 $\mathcal P^{-1}_\mathrm{NS}$ & n/a & 2.2 s (7) \\ 
 GD & n/a & $>$ 4h (20k) \\ 
 \hline
\end{tabular}
\vspace{-0.2cm}
\end{wraptable}


Comparing the inferred solutions from the network to an iterative approach shows a large difference in inference time (table~\ref{tab:fluid-speed}).
To reach the same solution quality as the neural network prediction, $\mathcal P^{-1}_\mathrm{NS}$ needs 7 iterations on average, which takes more than 10,000 times as long, and gradient descent (GD) does not reach the same quality even after 20k iterations.
This difference is caused by $P^{-1}_\mathrm{NS}$ having to run the full forward and backward simulation for each iteration.
This cost is also required for each training iteration of the network but once converged, its inference is extremely fast, solving around 9000 problems per second in batch mode.
For both iterative solver and network, we used a batch size of 64 and divide the total time by the batch size.


\subsection{Limitations and Discussion}
While SIP training manages to find vastly more accurate solutions for the examples above, there are some caveats to consider.
First, an approximately scale-invariant physics solver is required. While in low-dimensional $x$ spaces Newton's method is a good candidate, high-dimensional spaces require another form of inversion.
Some equations can locally be inverted analytically but for complex problems, domain-specific knowledge may be required.
However, this is a widely studied field and many specialized solvers have been developed~\cite{QMC_book}.

Second, SIP uses traditional first-order optimizers to determine $\Delta\theta$.
As discussed, these solvers behave poorly in ill-conditioned settings which can also affect SIP performance when the network outputs lie on different scales.
Some recent works address this issue and have proposed network optimization based on inversion~\cite{HessianFree, kFAC, elias2020hessian}.

Third, while SIP training generally leads to more accurate solutions, measured in $x$ space, the same is not always true for the loss $L = \sum_i L_i$.
SIP training weighs all examples equally, independent of the curvature $|\frac{\partial^2L}{\partial x^2}|$ near a chosen solution.
This can cause small errors in examples with large curvatures to dominate $L$.
In these cases, or when the accuracy in $x$ space is not important, like in some control tasks, traditional training methods may perform better than SIP training.

\section{Conclusions}
We have introduced scale-invariant physics (SIP) training, a novel neural network training scheme for learning solutions to nonlinear inverse problems.
SIP training leverages physics inversion to compute scale-invariant updates in the solution space.
It provably converges assuming enough network updates $\Delta\theta$ are performed per solver evaluation and we have shown that it converges with a single $\Delta\theta$ update for a wide range of physics experiments.
The scale-invariance allows it to find solutions exponentially faster than traditional learning methods for many physics problems while keeping the computational cost relatively low.
While this work targets physical processes, SIP training could also be applied to other coupled nonlinear optimization problems, such as differentiable rendering or training invertible neural networks.

Scale-invariant optimizers, such as Newton's method, avoid many of the problems that plague deep learning at the moment.
While their application to high-dimensional parameter spaces is currently limited, we hope that our method will help establish them as commonplace tools for training neural networks in the future.

%% file: appendix.tex
\section{Convergence} \label{app:convergence-proof}
Here we give a more detailed explanation of some of the convergence arguments made in the paper.

\subsection{Assumptions}
The following is a sufficient, self-contained list of all assumptions required for the theorems in the main text.

\begin{enumerate}
    \item Let $L = \{L_i: \mathbb R^d \rightarrow \mathbb R\, |\, i = 1, ..., N\}$ be a finite set of functions and let $x_i^* \in \mathbb R^d$ be any global minima of $L_i$.
    \item Let $U = \{U_i: \mathbb R^d \rightarrow \mathbb R^d\, |\, i = 1, ..., N\}$ be a set of update functions for $L_i$ such that $\exists \tau > 0 \ : \ \forall x \in \mathbb R^d \ : \ L_i(x) - L_i(x + U_i(x)) \geq \tau \left(L(x) - L(x^*)\right)$
    and $\exists K > 0 \ : \ \forall x \in \mathbb R^d \ : \ ||U_i(x)|| \leq K (L(x) - L(x^*))$.
    \item Let $\mathrm{NN}_\theta: \mathbb R^m \rightarrow \mathbb R^d$ be differentiable w.r.t. $\theta$ with the property that $\exists \eta > 0 \, : \, \forall i \in 1, ..., N \  \forall x_i \in \mathbb R^d \ \forall \epsilon > 0 \ \exists n \in \mathbb N : || \mathrm{NN}_{\theta_n^{x_i}}(y_i) - x_i ||_2 \leq \epsilon$ where $\theta_n^x$ is the sequence of gradient descent steps with $\theta_{n+1}^x = \theta_n^x - \eta  \left(\frac{\partial \mathrm{NN}_\theta}{\partial \theta}\right)^T(\mathrm{NN}_{\theta_n} - x)$, $\eta > 0$.
\end{enumerate}

Given these assumptions, SIP training with sufficiently many network optimization steps $n$ is guaranteed to converge to an optimum point or region.
In the two special cases mentioned in the main text, it is guaranteed to converge even for $n=1$.

If we are interested in convergence to a global optimum, we may additionally require that $L$ is convex.

\subsection{Probable Loss Decrease for the Case $n = 1$}

\begin{wrapfigure}{R}{5cm}
\centering
\vspace{-5mm}
\includegraphics[width=5cm]{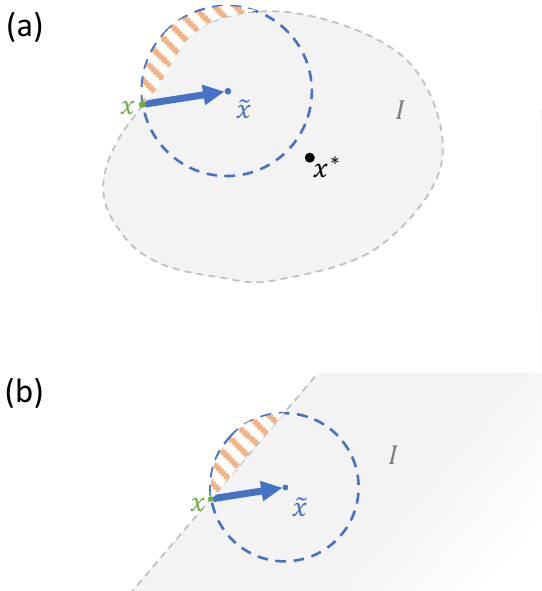}
\caption{
    (a)~Convergence visualization for the probable loss decrease in the case $n=1$.
    (b)~Zoomed view for small $\eta$.
    All shown objects are visualized in $x$ space for one example $i$. 
    }
\label{fig:convergence2}
\vspace{-.4cm}
\end{wrapfigure}

To show that the loss decreases on average, we require the the assumption that an update of the neural network weights $\theta$ does not systematically distort the update direction in $x$.
More formally, we update $\theta$ to minimize 
$$\tilde L = \frac 1 2 || \mathrm{NN}_{\theta_n^{x_i}}(y) - \tilde x ||_2^2,$$
where $\tilde x$ denotes the target in $x$ space which is computed using $U$ as defined above.
This can be done with gradient descent,
$$\Delta\theta = - \eta  \left(\frac{\partial \mathrm{NN}_\theta}{\partial \theta}\right)^T(\mathrm{NN}_\theta - \tilde x),$$
or any other method.
In fact, recent works have shown that the Jacobian $\frac{\partial \mathrm{NN}_\theta}{\partial \theta}$ can be inverted numerically to compute more precise $\theta$ updates~\cite{schnell2022half}.
The updated network will predict a new vector $\mathrm{NN}_{\theta+\Delta\theta}(y)$ and we will denote the shift in $x$ space resulting from the update $\Delta\theta$ as $\Delta\mathrm{NN} \equiv \mathrm{NN}_{\theta+\Delta\theta}(y) - \mathrm{NN}_\theta(y)$.
This shift is guaranteed to lie closer to $\tilde x$ (within the blue circle in Fig.~\ref{fig:convergence2}) as long as we choose $\eta$ appropriately.
We also know that at $\tilde x$ itself, as well as in a small region around $\tilde x$, the loss $L$ is smaller than $L(x)$.
We denote the region of decreased loss $I = \{x \in \mathbb R^d \ : \ L(x) < L(x_n)\}$.

Note that the circle radius $||\tilde x - x||_2$ scales with the step size of the higher-order optimizer since $\tilde x = x + U(x)$.
Factoring out this step size and integrating it into $\eta$ allows us to arbitrarily scale our problem in $x$ space.
In particular, when choosing $\eta$ to be small, we can linearly approximate the loss landscape (Fig.~\ref{fig:convergence2}b).
Its boundary becomes a straight line separating increased and decreased regions of the loss, and must cut the circle so that more than half of it lies within $I$.
More importantly, the expectation of $L$ integrated over the circle is smaller than $L(x)$.

Therefore, if $\Delta\mathrm{NN}$ points towards $\tilde x$ on average, i.e. $\mathbb E \left[ \Delta\mathrm{NN} \right] = \lambda (\tilde x - x)$ 
with $\lambda \in (0, 1]$, and we choose $\eta$ small enough, the loss must decrease on average.
%
While the assumption that updates of $\theta$ do not systematically distort the  direction in $x$ may not hold for all problems or network architectures, we have empirically observed it to be true over a wide range of tests.


\section{Detailed Description of the Experiments} \label{app:experiments}
Here, we give a more detailed description of our experiments including setup and analysis.
The implementation of our experiments is based on the $\Phi_\textrm{Flow}$ (PhiFlow) framework~\cite{phiflow} and uses TensorFlow~\cite{tensorflow} and PyTorch~\cite{PyTorchAutoDiff} for automatic differentiation.
All experiments were run on an NVidia GeForce RTX 2070 Super.
Our code is open source, available at \url{https://github.com/tum-pbs/SIP}.

In all of our experiments, we compare SIP training to standard neural network optimizers.
To make this comparison as fair as possible, we try to avoid other factors that might prevent convergence, such as overfitting.
Therefore, we perform our experiments with effectively infinite training sets, sampling target observations $y^*$ on-the-fly.
This setup ensures that all optimizers can, in principle, converge to zero loss.
Since a numerical simulation of the forward problem is assumed to be available, unlimited amounts of synthetic training data can always be generated, further justifying this assumption.
To check for possible biases introduced by this procedure, we also performed our experiments with finite training sets but did not observe a noticeable change in performance as long as the data set sizes were reasonable, e.g. 1000 samples for the characterization (sine) experiment.

\subsection{Characterization with two-dimensional sine function} \label{app:sin}
In this experiment, we consider the nonlinear process
$$\mathcal P(x) = \left(\frac{\sin(\hat x_1)}{\xi}, \xi \cdot \hat x_2 \right)
\quad \mathrm{with} \quad x
\hat x = \gamma \cdot R_\phi \cdot x
\quad \mathrm{and} \quad
R_\phi = 
\begin{pmatrix}
\cos(\phi) & -\sin(\phi) \\
\sin(\phi) & \cos(\phi)
\end{pmatrix}$$
where $\gamma$ denotes a global scaling factor.
We set $\gamma=10$ as this value leads to the fastest convergence when training the network with traditional optimizers like Adam.

The conditioning factor $\xi$ continuously scales any compact solution manifold, elongating it along one axis and compressing it along the other.
Therefore, the range of values required for a solution manifold scales linearly with $\xi$, in both $x_1$ and $x_2$ if $\phi \neq 0$.
To take this into account, we show the relative accuracy $\frac{\mathrm{||x-x^*||_2}}{\xi}$ in Fig.~\ref{fig:sin-main}d, as described in the main text.
To measure $||x-x^*||_2$, we find the true solution $x^*$ analytically using $\sin^{-1}$ and determining which solution is closest.

\paragraph{Saddle-free Newton}
For SIP training, we use a variant of Newton's method.
Without modification, Newton's method also approaches maxima and saddle points which exist in this experiment.
To avoid these, we project the Newton direction and gradient into the eigenspace of the Hessian, similar to~\cite{SaddleFreeNewton}.
There, we can easily flip the Newton update along eigenvectors of the Hessian to ensure that the update always points in the direction of decreasing loss.


\paragraph{Neural network training}
We sample target states $y^* = (y_1^*, y_2^*)$ uniformly in the range $[-1, 1]$ each, and feed $y^*$ to a fully-connected neural network that predicts $x = (x_1, x_2)$.
The network consists of three hidden layers with 32, 64 and 32 neurons, respectively, and uses ReLU activations.
The final layer has no activation function and outputs two values that are interpreted as $x$.
We always train the network with a batch size of 100 and chose the best learning rate for each optimizer.
We determined the following learning rates to work the best for this problem:
SIP training with Adam $\eta=10^{-3}$,
Adam $\eta=10^{-3}$,
SGD $\eta=10^{-2} / \xi^2$,
Adadelta $\eta=3 \cdot 10^{-3}$,
Adagrad $\eta=3 \cdot 10^{-3}$,
RMSprop $\eta = 3 \cdot 10^{-5}$.
For Fig.~\ref{fig:sin-main}d and e, we run each optimizer for 10 minutes for each sample point which is enough time for more than 50k iterations with all traditional optimizers and around 15k iterations with SIP training.
We then average the last 10\% of recorded distances for the displayed value.

\subsection{Poisson's equation} \label{app:poisson}
We consider Poisson's equation, $\nabla^2 y = x$ where $x$ is the initial state and $y$ is the output of the simulator.
We set up a two-dimensional simulation with 80 by 60 cubic cells.
Our simulator computes $y = \mathcal P(x) = \nabla^{-2} x$ implicitly via the conjugate gradient method.
The inverse problem consists of finding an initial value $x^*$ for a given target $y^*$ such that $\nabla^2 y^* = x^*$.
We formulate this problem as minimizing $L(x) = \frac 1 2 || \mathcal P(x) - y^* ||_2^2 = \frac 1 2 || \nabla^{-2}(x - x^*) ||_2^2$.
We now investigate the computed updates $\Delta x$ of various optimization methods for this problem.

\paragraph{Gradient descent}
Gradient descent prescribes the update $\Delta x = - \eta \cdot \left( \frac{\partial L}{\partial x} \right)^T
= - \eta \cdot \nabla^{-2} \left(y - y^* \right)$ which requires an additional implicit solve for each optimization step.
This backward solve produces much larger values than the forward solve, causing GD-based methods to diverge from oscillations unless $\eta$ is very small.
We found that GD requires $\eta \leq 2\cdot 10^{-5}$, while the momentum in Adam allows for larger $\eta$.
For both GD and Adam, the optimization converges extremely slowly, making GD-based methods unfeasible for this problem.

\paragraph{SIP Gradients via analytic inversion}
Poisson's equation can easily be inverted analytically, yielding $x = \nabla^2 y$.
Correspondingly, we formulate the update step as $\Delta x = - \eta \cdot \frac{\partial x}{\partial y} \cdot (y - y^*) = - \eta \cdot \nabla^2 \left(y - y^* \right)$ which directly points to $x^*$ for $\eta = 1$.
Here the Laplace operator appears in the computation of the optimization direction.
This is much easier to compute numerically than the Poisson operator used by gradient descent.
Consequently, no additional implicit solve is required for the optimization and the cost per iteration is less than with gradient descent.
This computational advantages also carries over to neural network training where this method can be integrated into the backpropagation pipeline as a gradient.

\paragraph{Neural network training}
\begin{figure*}
\centering
\includegraphics[width=1.0\textwidth]{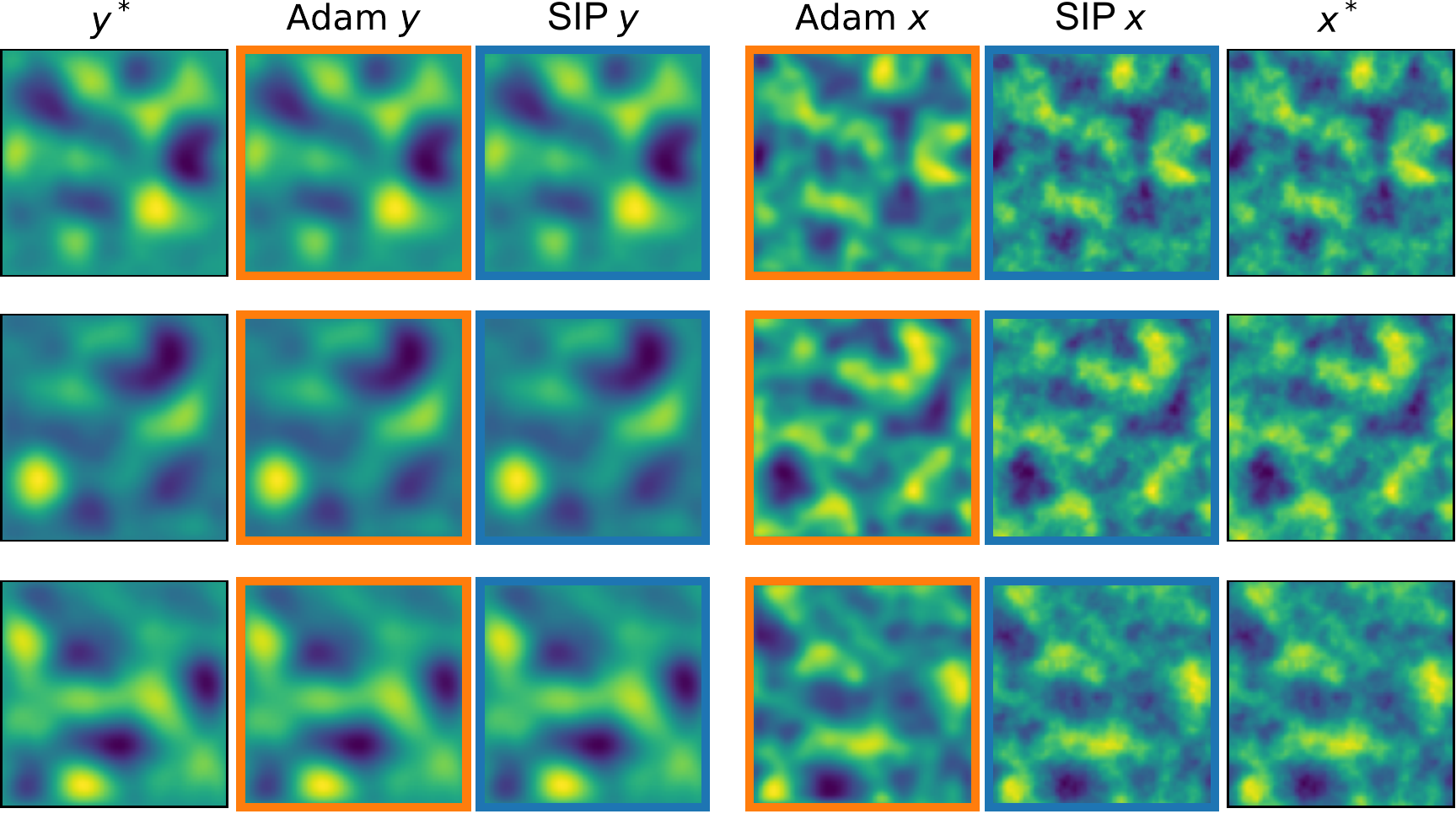}
\includegraphics[width=1.0\textwidth]{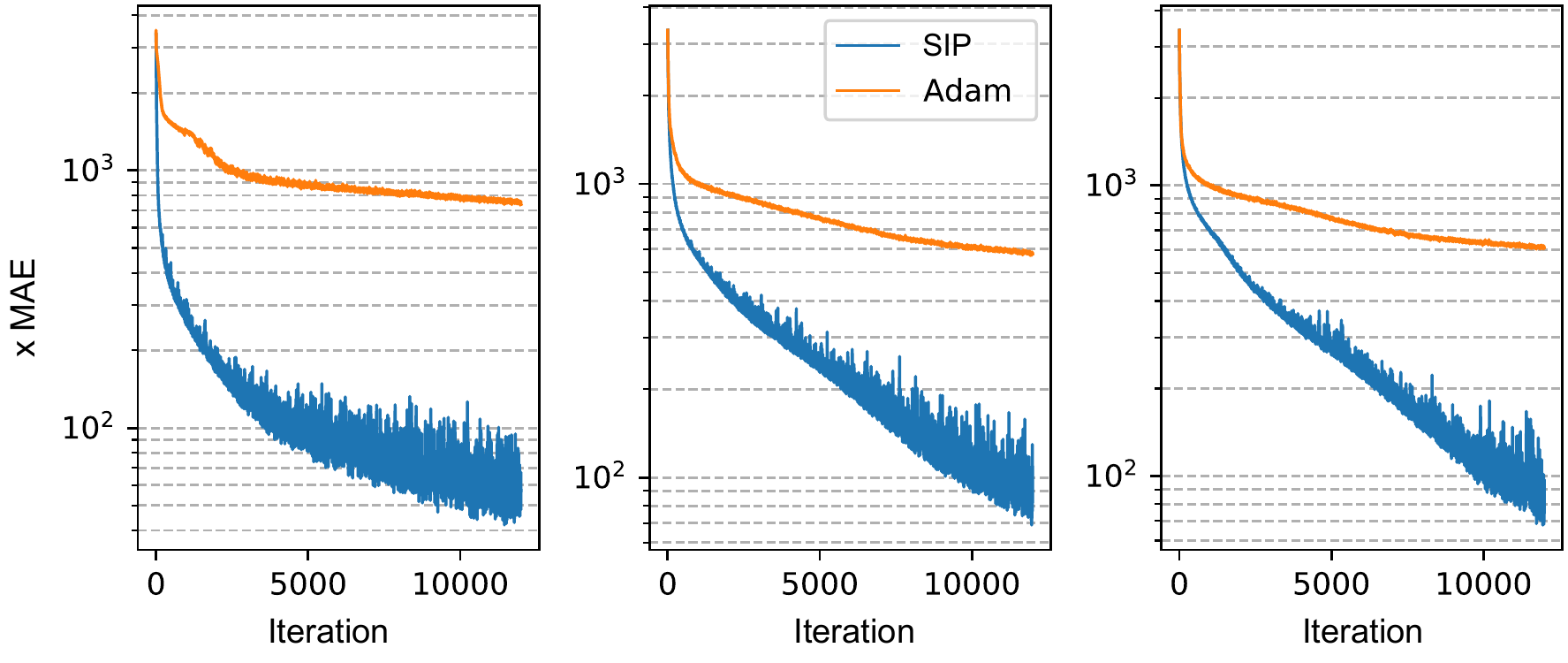}
\caption{\label{fig:poisson} Inverse problems involving Poisson's equation. \textbf{Top:}~Three examples from the data set, from left to right: observed target ($y^*$),
simulated observations resulting from network predictions (Adam $y$, SIP $y$), predicted solutions (Adam $x$, SIP $x$), ground truth solution ($x^*$). Networks were trained for 12k iterations.
\textbf{Bottom:}~Neural network learning curves for three random network initializations, measured as $||x-x^*||_1$.}
\end{figure*}

We first generate ground truth solutions $x^*$ by adding fluctuations of varying frequencies with random amplitudes.
From these $x^*$, we compute $y^* = \mathcal P(x^*)$ to form the set of target states $\mathcal Y$.
This has the advantage that learning curves are representative of both test performance as well as training performance.
The top of Fig.~\ref{fig:poisson} shows some examples generated this way.
All training methods \rebuttal{except for the Fourier neural operators (FNO)} use a a U-net~\cite{UNet} with a total of 4 resolution levels and skip connections.
The network receives the feature map $y^*$ as input.
Max pooling is used for downsampling and bilinear interpolation for upsampling.
After each downsampling or upsampling operation, two blocks consisting of 2D convolution with kernel size of 3x3, batch normalization and ReLU activation are performed.
For training with AdaHessian, we also tested $\mathrm{tanh}$ for activation but observed no difference in performance.
All of these convolutions output 16 feature maps and a final 1x1 convolution brings the output down to one feature map.
The network contains a total of 37,697 trainable parameters.
We use a mini-batch size of 128 for all methods.

For SGD and Adam training, the composite gradient of $\mathrm{NN} \circ \mathcal P$ is computed with TensorFlow or PyTorch, enabling an end-to-end optimization.
The learning rate is set to $\eta = 10^{-3}$ with Adam and $\eta = 3\cdot 10^{-12}$ for SGD.
The extremely small learning rate for SGD is required to balance out the large gradients and is consistent with the behavior of gradient-descent optimization on single examples where an $\eta = 2\cdot 10^{-5}$ was required.
We use a typical value of 0.9 for the momentum of SGD and Adam.

For training with AdaHessian~\cite{yao2020adahessian}, we use the implementation from \texttt{torch-optimizer} \cite{Novik_torchoptimizers} and found the best learning rate to be $\eta = 10^{-6}$ which, similar to SGD, is much smaller than what is used on well-conditioned problems.
However for $\eta >= 10^{-4}$ AdaHessian diverges on the Poisson equation, independent of which activation function is used.
The Hessian power is another hyperparameter of the AdaHessian optimizer and we set it to the standard value of 0.5 which converges much better than a value of 1.0.

\rebuttal{For training with the Hessian-free optimizer~\cite{HessianFree}, we use the \texttt{PyTorchHessianFree} implementation from \url{https://github.com/ltatzel/PyTorchHessianFree}.
We use the generalized Gauss-Newton matrix to approximate the local curvature because it yields more stable results than the Hessian approximation in our tests.
The optimizer uses a learning rate of $\eta = 1.0$ and adaptive damping during training.
On this problem, the Hessian-free optimizer takes between 300 and 1700 seconds to compute a single update step.
This is largely due to the CG loop, performing between 50 and 250 iterations for each update.
The solution error drops to 900 within the plotted time frame in Fig.~\ref{fig:Poisson}b, corresponding to less then 6 iterations.
Convergence speed then slows continuously, reaching a MAE of 500 after about 6 hours.}

For the training using Adam with SIP gradients, we compute $\Delta x$ as described above and keep $\eta = 10^{-3}$.
For each case, we set the learning rate to the maximum value that consistently converges.
The learning curves for three additional random network initializations are shown at the bottom of Fig.~\ref{fig:poisson}, while Fig.~\ref{fig:step-times} shows the computation time per iteration.

\rebuttal{
Training a Fourier neural operator network~\cite{kovachki2021neural} on the same task requires a different network architecture.
We use a standard 2D FNO architecture with width 32 and 12 modes along x and y.
This results in a much larger network consisting of 1,188,353 parameters.
Despite its size, its performance measured against wall clock time is superior to that of smaller versions we tested.
We train the FNO using Adam with $\eta=0.003$, which yields the best results, converging in a stable manner.}

\subsection{Heat equation} \label{app:heat}
We consider a two-dimensional system governed by the heat equation $\frac{\partial u}{\partial t} = \nu \cdot \nabla^2 u$.
Given an initial state $x = u_0$ at $t_0$, the simulator computes the state at a later time $t_*$ via $y = u_* = \mathcal P(x)$.
Exactly inverting this system is only possible for $t \cdot \nu = 0$ and becomes increasingly unstable for larger $t \cdot \nu$ because initially distinct heat levels even out over time, drowning the original information in noise.
Hence the Jacobian of the physics $\frac{\partial y}{\partial x}$ is near-singular.
In our experiment we set $t\cdot \nu = 8$ on a domain consisting of 64x64 cells of unit length.
This level of diffusion is challenging, and diffuses most details while leaving the large-scale structure intact.

We apply periodic boundary conditions and compute the result in frequency space where the physics can be computed analytically as $\hat y = \hat x \cdot e^{-k^2 (t_* - t_0)}$ where $\hat y_k \equiv \mathcal F(y)_k$ denotes the $k$-th element of the Fourier-transformed vector $y$.
Here, high frequencies are dampened exponentially.
The inverse problem can thus be written as minimizing $L(x) = \frac 1 2 || \mathcal P(x) - y^* ||_2^2 = \frac 1 2 || \mathcal F^{-1}\left( \mathcal F(x) \cdot e^{-k^2 (t_* - t_0)} \right) - y^* ||_2^2$.

\paragraph{Gradient descent}
Using the analytic formulation, we can compute the gradient descent update as
$$\Delta x = - \eta \cdot \mathcal F^{-1}\left( e^{-k^2 (t_* - t_0)} \mathcal F(y - y^*) \right).$$
GD applies the forward physics to the gradient vector itself, which results in updates that are stable but lack high frequency spatial information.
Consequently, GD-based optimization methods converge slowly on this task after fitting the coarse structure and have severe problems in recovering high-frequency details.
This is not because the information is fundamentally missing but because GD cannot adequately process high-frequency details.

\paragraph{Stable SIP gradients}
The frequency formulation of the heat equation can be inverted analytically, yielding
$\hat x_k = \hat y_k \cdot e^{k^2 (t_* - t_0)}$.
This allows us to define the update 
$$\Delta x = - \eta \cdot \mathcal F^{-1}\left( e^{k^2 (t_* - t_0)} \mathcal F(y - y^*) \right).$$
Here, high frequencies are multiplied by exponentially large factors, resulting in numerical instabilities.
When applying this formula directly to the gradients, it can lead to large oscillations in $\Delta x$.
This is the opposite behavior compared to Poisson's equation where the GD updates were unstable and the SIP stable.

The numerical instabilities here can, however, be avoided by taking a probabilistic viewpoint.
The observed values $y$ contain a certain amount of noise $n$, with the remainder constituting the signal $s = y - n$.
For the noise, we assume a normal distribution $n \sim \mathcal N(0, \epsilon \cdot y)$ with $\epsilon > 0$ and for the signal, we assume that it arises from reasonable values of $x$ so that $s \sim \mathcal N(0, \delta \cdot e^{-k^2})$ with $\delta > 0$.
Then we can estimate the probability of an observed value arising from the signal using Bayes' theorem
$p(s | y) = \frac{p(y \,|\, s) \cdot p(s)}{p(y \,|\, s) \cdot p(s) + p(y \,|\, n) \cdot p(n)}$ where we assume the priors $p(s) = p(n) = \frac 1 2$.
Based on this probability, we dampen the amplification of the inverse physics which yields a stable inverse.
Gradients computed in this way hold as much high-frequency information as can be extracted given the noise that is present.
This leads to a much faster convergence and more precise solution than any generic optimization method.

\paragraph{Neural network training}
\begin{figure*}
\centering
\includegraphics[width=1.0\textwidth]{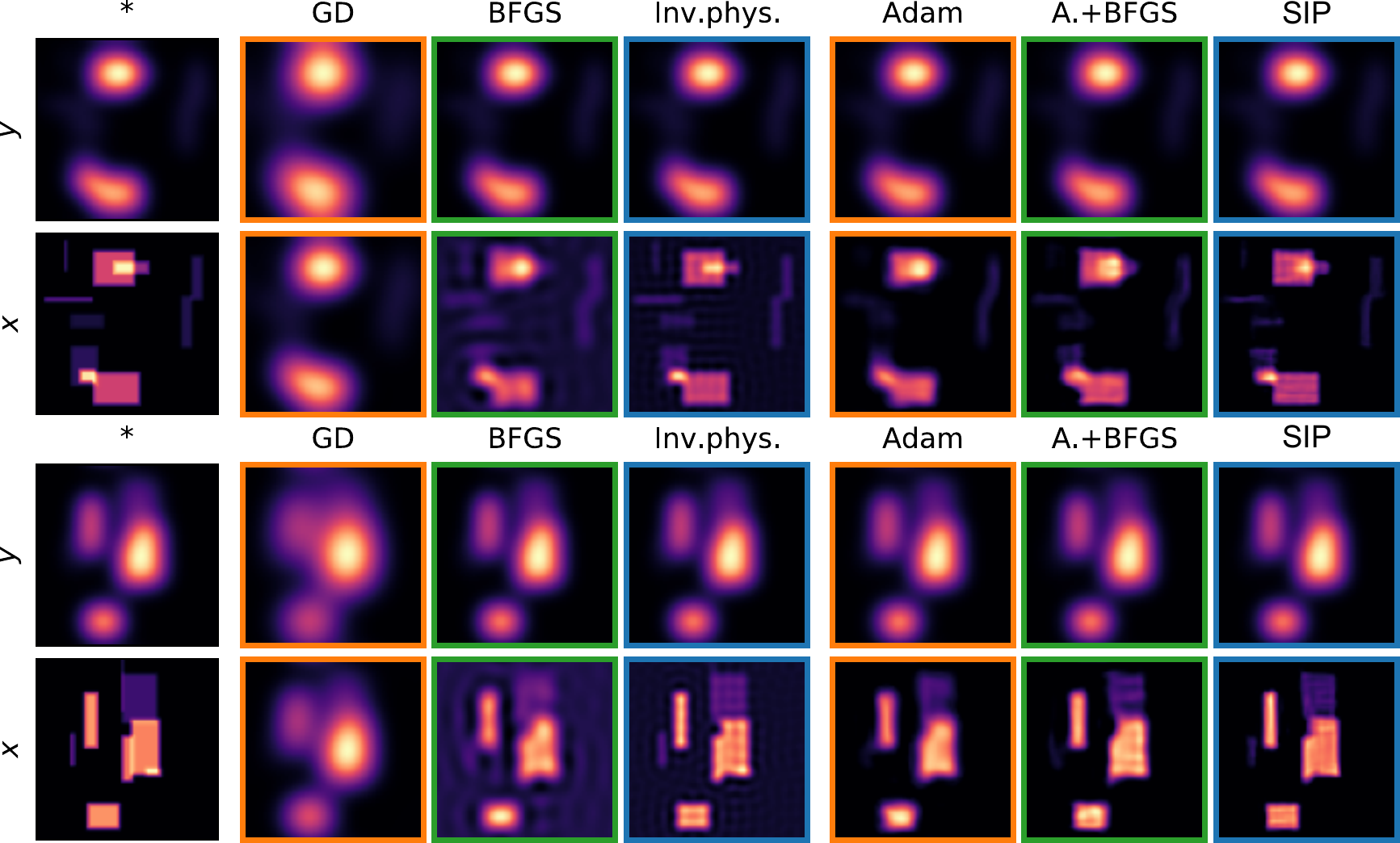}
\includegraphics[width=1.0\textwidth]{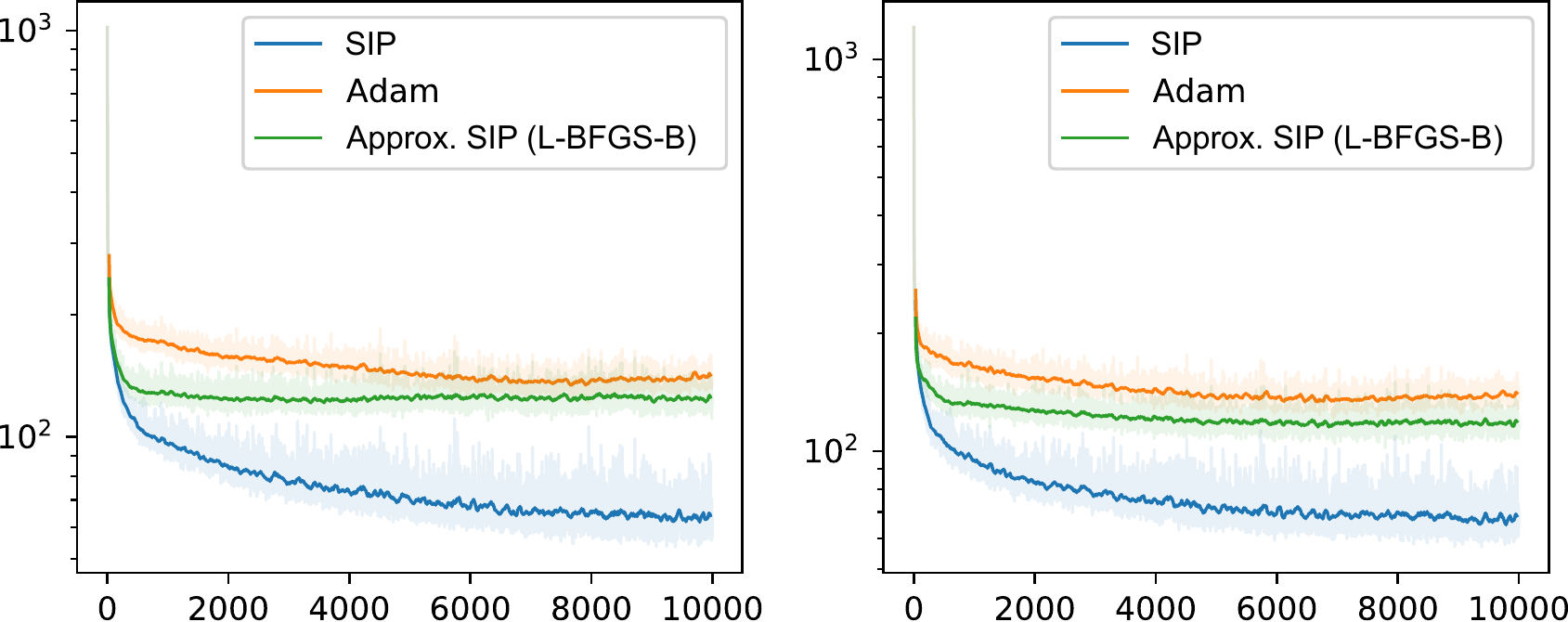}
\caption{\label{fig:heat} Inverse problems involving the heat equation. \textbf{Top:}~Two examples from the data set. The top row shows observed target ($y^*$) and simulated observations resulting from inferred solutions.
The bottom row shows the ground truth solution ($x^*$) and inferred solutions.
From left to right: ground truth; gradient descent (GD), L-BFGS-B (BFGS) and inverse physics (Inv.phys.), running for 100 iterations each, starting with $x_0=0$; Networks trained for 10k iterations.
\textbf{Bottom:}~Neural network learning curves for two random network initializations, measured in terms of $||x-x^*||_1$.}
\end{figure*}

For training, we generate $x^*$ by randomly placing between 4 and 10 hot rectangles of random size and shape in the domain and computing $y = \mathcal P(x^*)$.
\rebuttal{For the neural network, we use the same U-net and FNO architectures as in the previous experiment.
We train all methods with a batch size of 128.
A learning rate of $\eta = 10^{-3}$ yields the best results for SGD, Adam, AdaHessian, FNO and SIP training.}
Unlike the Poisson experiment, where such large learning rates lead to divergence due to the large gradient magnitudes, the heat equation produces relatively small and predictable gradients.
Larger $\eta$ can still result in divergence, especially with AdaHessian.
\rebuttal{For the Hessian-free optimizer, we use $\eta = 1.0$ with adaptive damping during training. Due to the more compute-intensive updates, we plot the running average over 8 mini-batches for Hessian-free in Fig.~\ref{fig:Poisson}, instead of the usual 64.}

Fig.~\ref{fig:heat} shows two examples from the data set, along with the corresponding inferred solutions, as well as the network learning curves for two network initializations.
The measured computation time per iteration is shown in Fig.~\ref{fig:step-times}.

\rebuttal{We perform additional hyperparameter studies on this experiment to better gauge how SIP training compares to traditional network optimization in a variety of settings.
Fig.~\ref{fig:heat-hyperparameters} shows the learning curves for a varying learning rate $\eta$ and batch size $b$.
The best performance of both methods is achieved at $\eta = 10^{-3}$ and $b=128$.

Additionally, we evaluate the methods on finite data sets with sizes between 32 and 2048 examples while also varying the batch size.
Fig.~\ref{fig:heat-finite-data-sets} shows the performance on both the training and the test set.
Both SIP training and Adam exhibit overfitting under the same conditions.
For data set sizes of 128 and less, both methods start to overfit early on.
This is especially pronounced for large batch sizes where less randomness is involved.
In the cases where the batch size is equal to or larger than the data set size, no mini-batching is used. 
We tile the data set where needed.
For very small data sets, SIP training seems to exhibit stronger overfitting but this can be explained by its superior performance.
The test performance of SIP training always surpasses the Adam variant, even on small data sets.
}

\begin{figure*}
\centering
\includegraphics[width=1.0\textwidth]{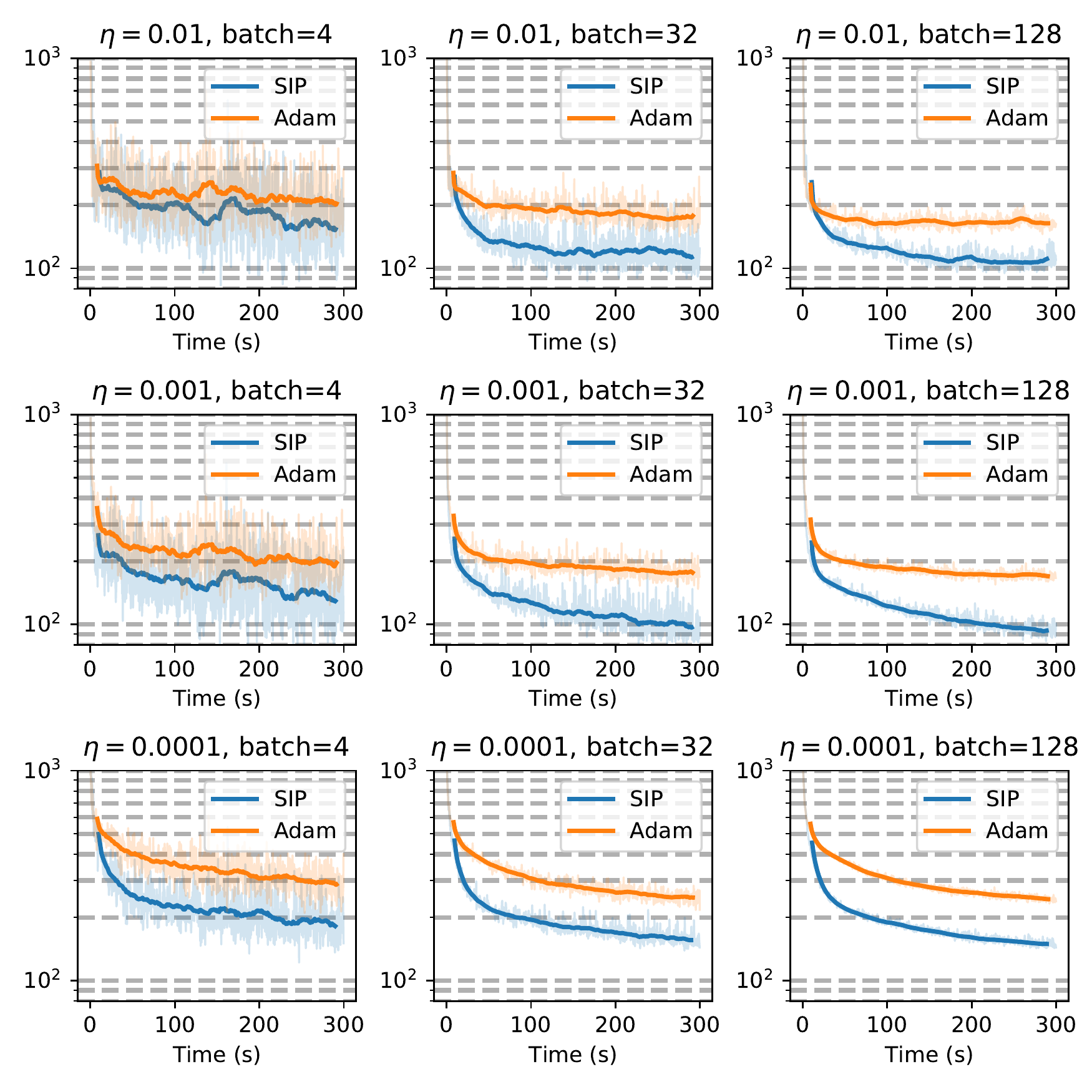}
\caption{\label{fig:heat-hyperparameters} Hyperparameter study for the heat equation experiment. The learning rate $\eta$ varies vertically from $10^{-4}$ to $10^{-2}$ and the batch size varies horizontally from 4 to 128 examples per mini-batch. The solid curves are averaged over 64 mini-batches. A learning rate of $10^{-3}$ with large batch sizes yields best performance.}
\end{figure*}

\begin{figure*}
\centering
\includegraphics[width=1.0\textwidth]{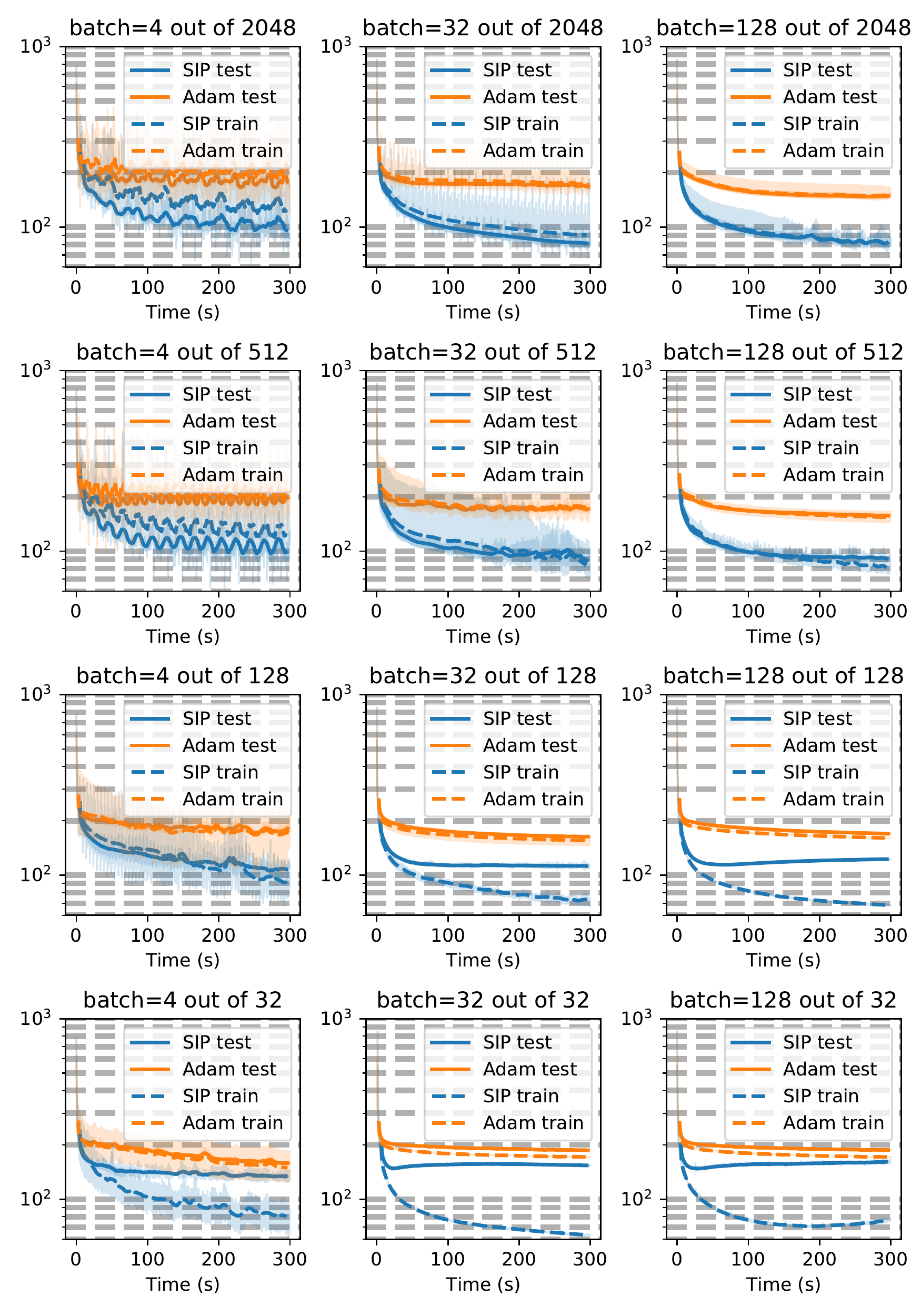}
\caption{\label{fig:heat-finite-data-sets} \rebuttal{Heat equation experiment with different training set sizes. The training set size $|\mathcal T|$ varies vertically from 32 to 2048 and the batch size $b$ varies horizontally from 4 to 128 examples per mini-batch. The test set is of fixed size 128. The solid curves are averaged over 64 mini-batches. Overfitting occurs in both SIP training and Adam when no mini-batches are used, $|\mathcal T| \leq b$, and for small $b$ where the convergence is unstable.}}
\end{figure*}

\rebuttal{We also run iterative optimizers on individual examples of the data set.
Fig.~\ref{fig:heat-iterative-opt} shows the optimization curves of gradient descent and \mbox{L-BFGS-B} and compares them to the network predictions.
\mbox{L-BFGS-B} converges faster than gradient descent but both iterative optimizers progress slowly on this inverse problem due to its ill-conditioned nature.
After 500 iterations, \mbox{L-BFGS-B} matches the solution accuracy of the neural network trained with Adam.
We have run both traditional optimizers for 1000 iterations, representing 102 seconds for BFGS and 36 seconds for gradient descent.
This is about 1000 times longer than the network predictions, which finish within 64 ms.
}

\begin{figure*}
\centering
\includegraphics{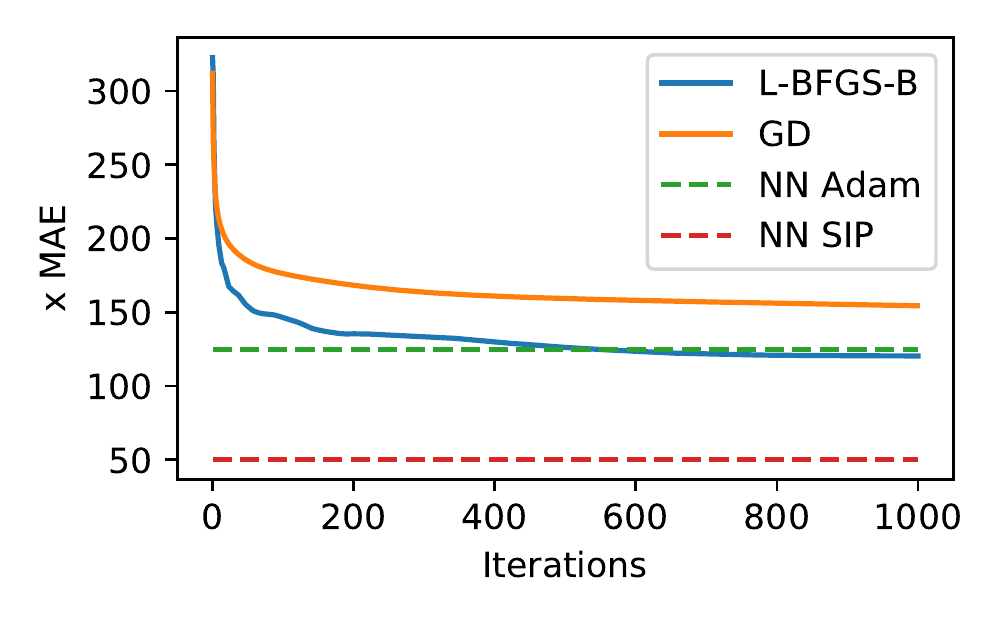}
\caption{\label{fig:heat-iterative-opt} \rebuttal{Iterative optimization of individual examples from a test set of 128 heat examples. Gradient descent (GD) and L-BFGS-B directly optimize the initial state $x_0$ (64x64 grid), independently for each example. The predictions of the trained and frozen neural networks (NN Adam, NN SIP) are evaluated on the same data set for reference.}}
\end{figure*}

\subsection{Navier-Stokes equations} \label{app:fluid}
Here, we give additional details on the simulation, data generation, SIP gradients and network training procedure for the fluid experiment.

\paragraph{Simulation details}
We simulate the fluid dynamics using a direct numerical solver.
We adopt the marker-in-cell (MAC) method~\cite{Mac1965, MAC} which guarantees stable simulations even for large velocities or time increments.
The velocity vectors are sampled in staggered form at the face centers of grid cells while the marker density is sampled at the cell centers.
The initial velocity $v_0$ is specified at cell centers and resampled to a staggered grid for the simulation.
Our simulation employs a second-order advection scheme~\cite{MacCormackStable} to transport both the marker and the velocity vectors.
This step introduces significant amount of numerical diffusion which can clearly be seen in the final marker distributions.
Hence, we do not numerically solve for adding additional viscosity.
Incompressibility is achieved via Helmholz decomposition of the velocity field using a conjugate gradient solve.

Neither pressure projection nor advection are energy-conserving operations.
While specialized energy-conserving simulation schemes for fluids exist~\cite{EnergyConservingFluid1999, EnergyConservingFluidAnim}, 
we instead enforce energy conservation by normalizing the velocity field at each time step to the total energy of the previous time step.
Here, the energy is computed as $E = \int_{\mathbb R^2} dx \, v(x)^2$ since we assume constant fluid density.

\paragraph{Data generation}
The data set consists of marker pairs $\{m_0, m_t\}$ which are randomly generated on-the-fly.
For each example, a center position for $m_0$ is chosen on a grid of 64x64 cells.
$m_0$ is then generated from discretized noise fluctuations to fill half the domain size in each dimension.
The number of marked cells is random.

Next, a ground truth initial velocity $v_0$ is generated from three components.
First, a uniform velocity field moves the marker towards the center of the domain to avoid boundary collisions. Second, a large vortex with random strength and direction is added. The velocity magnitude of the vortex falls off with a Gaussian function depending on the distance from the vortex center.
Third, smaller-scale vortices of random strengths and sizes are added additionally perturb the flow fields.
These are generated by assigning a random amplitude and phase to each frequency making up the velocity field.
The range from which the amplitudes are sampled depends on the magnitude frequency.

Given $m_0$ and $v_0$, a ground truth simulation is run for $t=2$ with $\Delta t = 0.25$.
The resulting marker density is then used as the target for the optimization.
This ensures that there exists a solution for each example.

\paragraph{Computation of SIP gradients}
To compute the SIP gradients for this example, we construct an explicit formulation $\hat v_0 = \mathcal P^{-1}(m_0, m_t \,|\,x_0)$ that produces an estimate for $v_0$ given an initial guess $x_0$ by locally inverting the physics.
From this information, it fits the coarse velocity, i.e. the uniform velocity and the vortex present in the data.
This use of domain knowledge, i.e., enforcing the translation and rotation components of the velocity field as a prior, is what allows it to produce a much better estimate of $v_0$ than the regular gradient.
More formally, it assumes that the solution lies on a manifold that is much more low-dimensional than $v_0$.
On the other hand, this estimator ignores the small-scale velocity fluctuations which limits the accuracy it can achieve.
However, the difficulty of fitting the full velocity field without any assumptions outweighs this limitation.
Nevertheless, GD could eventually lead to better results if trained for an extremely long time.

To estimate the vortex strength, the estimator runs a reverse Navier-Stokes simulation.
The reverse simulation is initialized with the marker $m_t^\textrm{rev} = m_t$ and velocity $v_t^\textrm{rev} = v^t$ from the forward simulation.
The reverse simulation then computes $m^\textrm{rev}$ and $v^\textrm{rev}$ for all time steps by performing simulation steps with $\Delta t = -0.25$.
Then, the update to the vortex strength is computed from the differences $m^\textrm{rev} - m$ at each time step and an estimate of the vortex location at these time steps.


\paragraph{Neural network training}
\begin{figure*}
\centering
\includegraphics[width=0.8\textwidth]{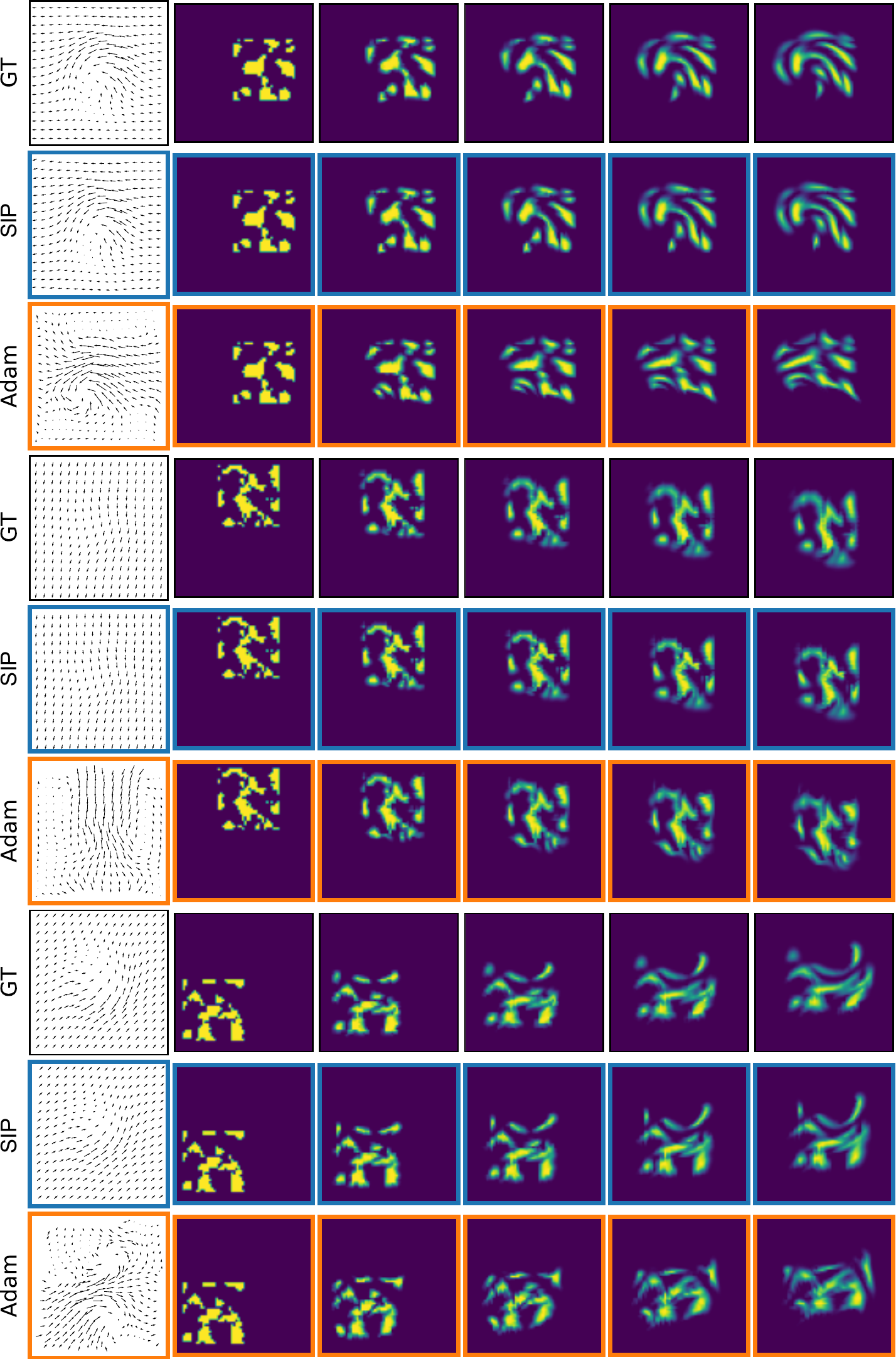}
\caption{\label{fig:fluid} Three example inverse problems involving the Navier-Stokes equations.
For each example, the ground truth (GT) and neural network reconstructions using Adam with SIP gradient (SIP) and pure Adam training (Adam) are displayed as rows.
Each row shows the initial velocity $v_0 \equiv x$ as well as five frames from the resulting marker density sequence $m(t)$, at time steps $t \in \{0, 0.5, 1, 1.5, 2 \}$.
The differences of the Adam version are especially clear in terms of $v_0$.
}
\end{figure*}

We train a U-net~\cite{UNet} similar to the previous experiments but with 5 resolution levels.
The network contains a total of 49,570 trainable parameters.
The network is given the observed markers $m_0$ and $m_t$, resulting in an input consisting of two feature maps.
It outputs two feature maps which are interpreted as a velocity field sampled at cell centers.

The objective function is defined as $|\mathcal F(\mathcal P(x) - y^*)| \cdot w$ where $\mathcal F$ denotes the two-dimensional Fourier transform and $w$ is a weighting vector that factors high frequencies exponentially less than low frequencies.

We train the network using Adam with a learning rate of 0.005 and mini-batches containing 64 examples each, using PyTorch's automatic differentiation to compute the weight updates.
We found that second-order optimizers like \mbox{L-BFGS-B} yield no significant advantage over gradient descent, and typically overshoot in terms of high-frequency motions. 
Example trajectories and reconstructions are shown in Fig.~\ref{fig:fluid} and performance measurements are shown in Fig.~\ref{fig:step-times}.

\begin{figure*}
\centering
\includegraphics[width=1.0\textwidth]{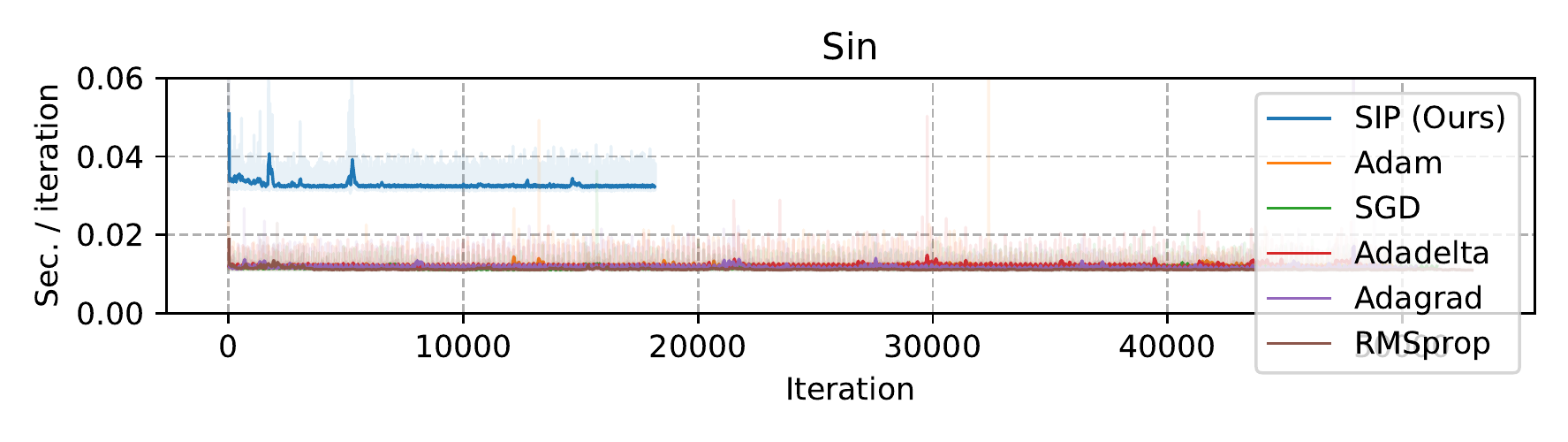}
\includegraphics[width=1.0\textwidth]{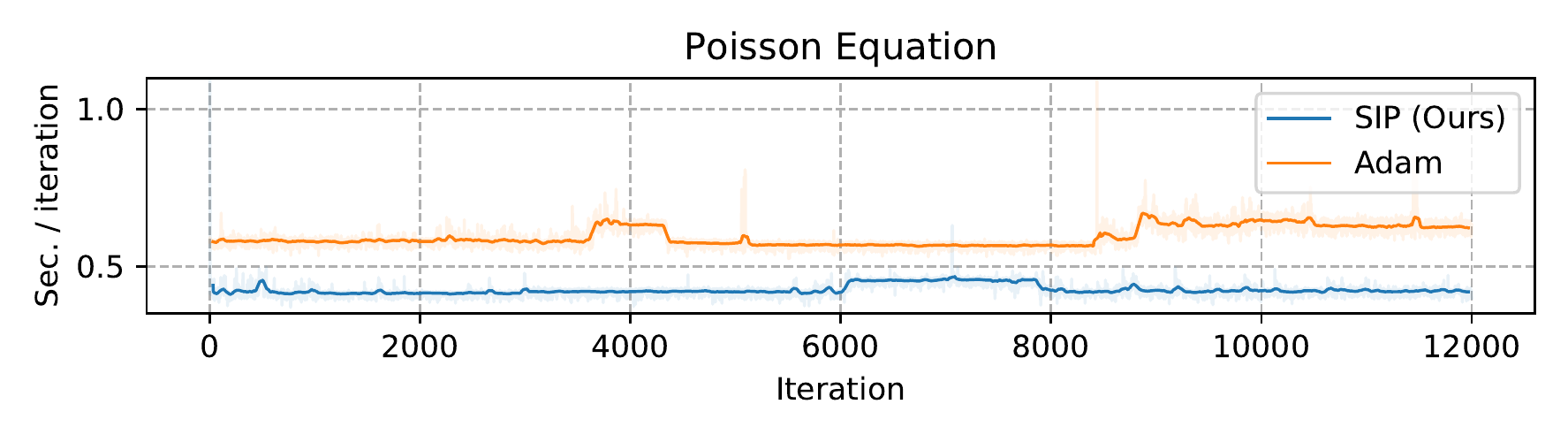}
\includegraphics[width=1.0\textwidth]{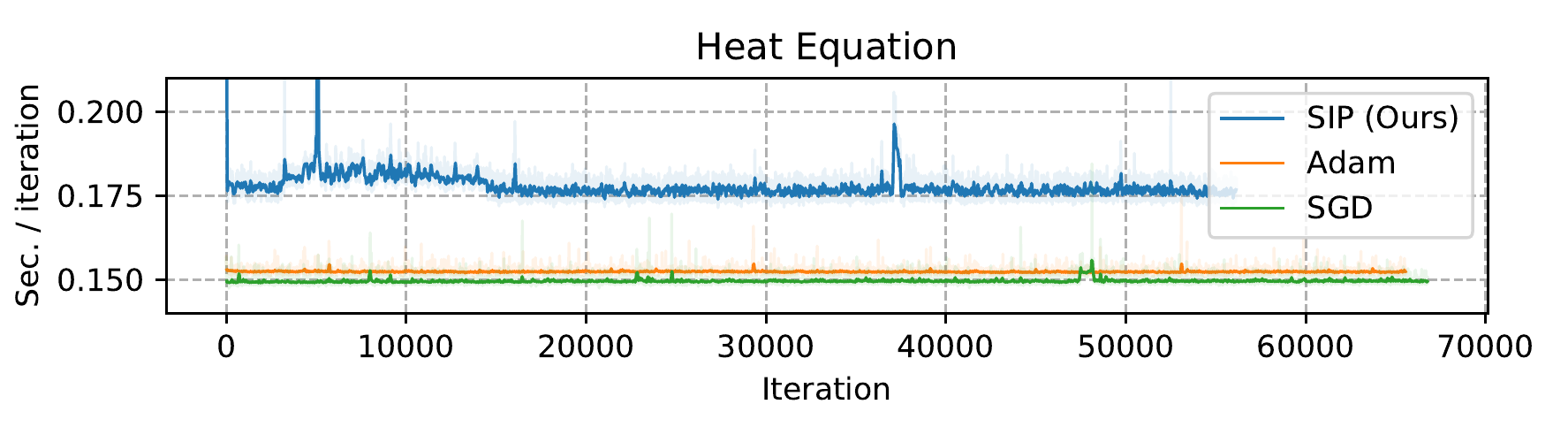}
\includegraphics[width=1.0\textwidth]{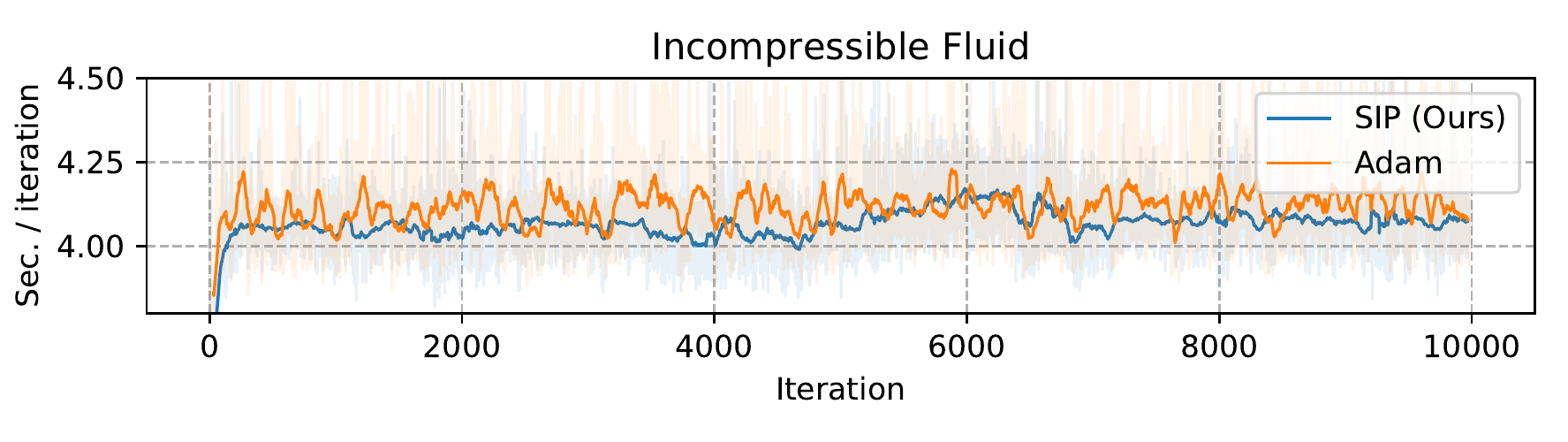}
\caption{\label{fig:step-times} Measured time per neural network training iteration for all experiments, averaged over 64 mini-batches. Step times were measured using Python's \texttt{perf\_counter()} function and include data generation, gradient evaluation and network update. In all experiments, the computational cost difference between the various gradients is marginal, affecting the overall training time by less than 10\%.}
\end{figure*}

\newpage
\section{Additional Experiments}

Here, we describe the additional experiments that were referenced in sections 1 and 2.

\subsection{Wave packet localization} \label{app:wavepacket}
As we state in the introduction, one advantage of solution inference using neural networks is that no initial guess needs to be provided for each problem.
Of course the network starts off with some initialization but we observe that the network can explore a much larger area of the solution space than any iterative solver.
To our knowledge this claim has not been verified as of yet.
However, as it is not directly relevant to our method, we do not discuss it in detail in the main text.
Instead, we provide a simple example here.

The wave packet localization experiment is an instance of a generic curve fitting problem.
The task is to find an offset parameter $t_0$ that results in least mean squared error between a noisy recorded curve and the model.

\paragraph{Data generation.}
We simulate an observed time series $y^*$ from a random ground truth position $x^* = t_0$.
Each time series contains 256 entries and consists of the wave packet and superimposed noise.
For the wave packet, we sample $t_0 \in [25.6, 128)$ from a uniform distribution.
The wave packet has the functional form
$$y(t) = A \cdot \sin(f \cdot (t-t_0)) \cdot \exp \left(- \frac 1 2 \frac{(t - t_0)^2}{\sigma^2} \right)$$
where we set $A = 1$, $f = 0.7$ and $\sigma = 20$ constant for all data.
For the noise, we superimpose random values sampled from the normal distribution $\mathcal N(0, 0.1)$ at each point.

\paragraph{Network architecture.}
We construct the neural network from convolutional blocks, followed by fully-connected layers, all using the ReLU activation function.
The input is first processed by five blocks, each containing a max pooling operation and two 1D convolutions with kernel size 3.
Each convolution outputs 16 feature maps.
The downsampled result is then passed to two fully connected layers with 64 and 32 and 2 neurons, respectively, before a third fully-connected layer produces the predicted $t_0$ which is passed through a Sigmoid activation function and normalized to the range of possible values.

\begin{figure*}
\centering
\includegraphics[width=1.0\textwidth]{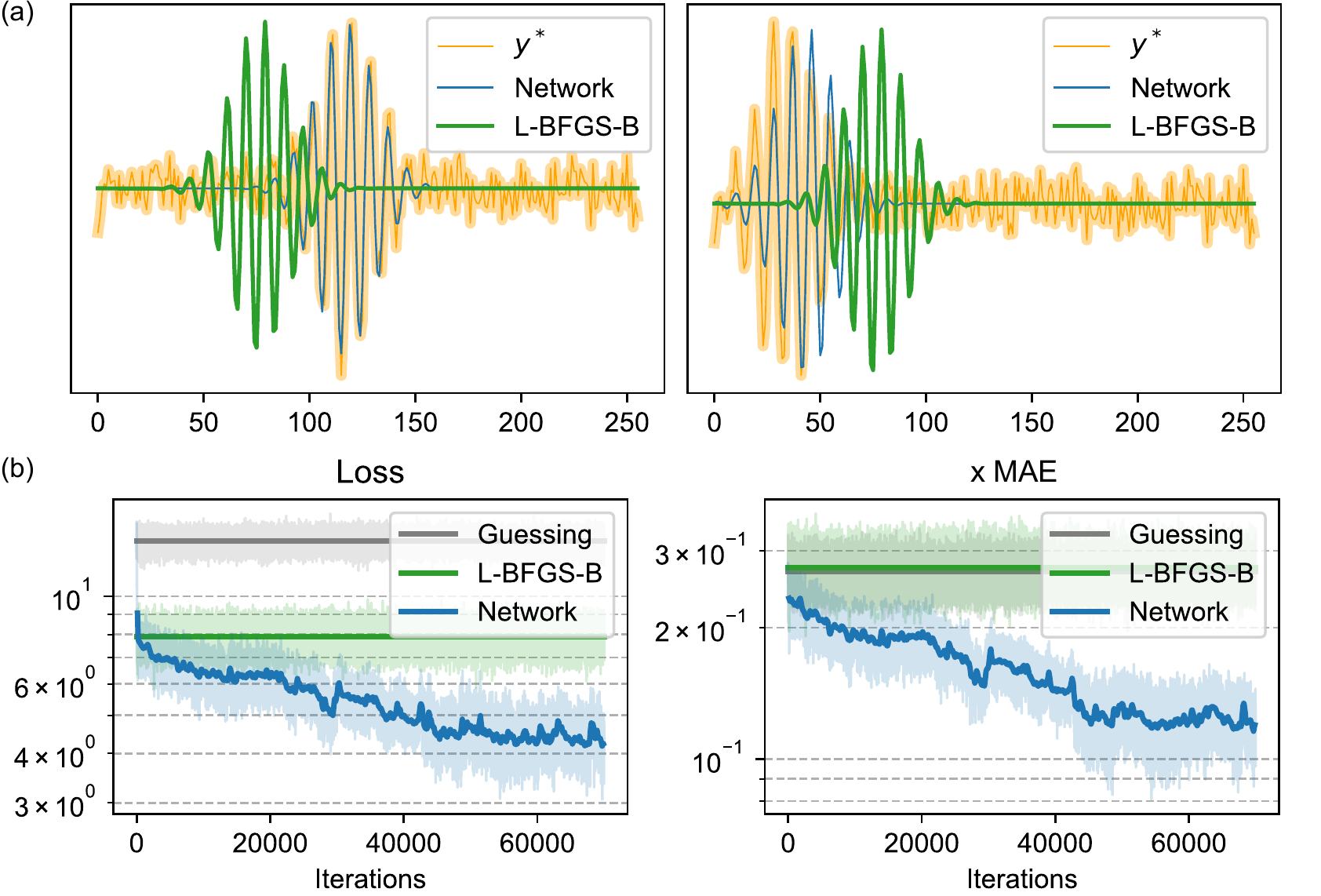}
\caption{
(a)~Two examples from the wave packet data set, each showing the simulated data (orange), neural network prediction and \mbox{L-BFGS-B} fit.
(b)~Learning curves of the network trained to localize wave packets. The performance of \mbox{L-BFGS-B} and random guessing are evaluated on the same data for reference. The left graph shows the objective $||y - y^*||_2^2$ and the right graph shows the $x$-space ($t_0$) deviation from the true solution.
}
\label{fig:wavepacket}
\end{figure*}

\paragraph{Training and fitting.}
We fit the data using \mbox{L-BFGS-B} with a centered initial guess ($t_0 = 76.8$) and the network output is offset by the same amount.
Both network and \mbox{L-BFGS-B} minimize the squared loss $||y(t_0) - y^*||_2^2$ and the resulting performance curves along with example fits are shown in Fig.~\ref{fig:wavepacket}.
We observe that \mbox{L-BFGS-B} manages to fit the wave packet perfectly when it is located very close to the center where the initial guess predicts it.
When the wave packet is located slightly to either side, \mbox{L-BFGS-B} gets stuck in a local optimum that corresponds to an integer phase shift.
When the wave packet is located further away from the initial guess, \mbox{L-BFGS-B} does not find it and instead fits the noise near the center.

The neural network is trained using Adam with learning rate $\eta = 10^{-3}$ and a batch size of 100.
Despite the simpler first-order updates, the network learns to localize most wave packets correctly, outperforming \mbox{L-BFGS-B} after 30 to 40 training iterations.
This improvement is possible because of the network's reparameterization of the problem, allowing for joint parameter optimization using all data.
When the prediction for one example is close to a local optimum, updates from different examples can prevent it from converging to that sub-optimal solution.

\paragraph{Conclusion}
This example shows that neural networks have the capability to explore the solution space much better than iterative solvers, at least in some cases.
Employing neural networks should therefore be considered, even when problems can be solved iteratively.

\newpage
\subsection{Gradient normalization for the exponential function} \label{app:exp}
The task in this simple experiment is to learn to invert the exponential function $\mathcal P(x) = e^x$.
As described in the text, both SGD and Adam converge very slowly on this task due to the gradients scaling linearly with $e^x$.

\paragraph{Gradient normalization}
We introduce a gradient normalization which first computes the gradient $\frac{\partial L}{\partial x}$.
It then normalizes this adjoint vector for each example in the batch to unit length, $\Delta x = \mathrm{sign}\left( \frac{\partial L}{\partial x} \right)$.
$\Delta x$ is then passed on to the network optimizer, replacing the standard adjoint vector for $x$.
Like with SIP training, we implement this using an $L_2$ loss for the effective network objective $\tilde L = \frac 1 2 || \mathrm{NN}(y^*) - (\mathrm{NN}(\circ) + \Delta x) ||_2^2$.

\paragraph{Neural network training}
For training data, we sample $x^*$ uniformly in the range $[-12, 0]$ and compute $y^* = e^{x^*}$.
We train a fully-connected neural network with three hidden layers, each using the Sigmoid activation function and consisting of 16, 64 and 16 neurons, respectively.
The network has a single input and output neuron.
We train the network for 10k iterations with each method, using the learning rates $\eta=10^{-3}$ for Adam, $\eta=10^{-2}$ for SGD and $\eta=10^{-3}$ for Adam with $x$ normalization.
Each mini-batch consists of 100 randomly sampled values.
